\documentclass{article}

\PassOptionsToPackage{numbers, compress}{natbib}
\usepackage[final]{neurips_2021}

\usepackage[utf8]{inputenc} % allow utf-8 input
\usepackage[T1]{fontenc}    % use 8-bit T1 fonts
\usepackage{hyperref}       % hyperlinks
\usepackage{url}            % simple URL typesetting
\usepackage{booktabs}       % professional-quality tables
\usepackage{amsfonts}       % blackboard math symbols
\usepackage{nicefrac}       % compact symbols for 1/2, etc.
\usepackage{microtype}      % microtypography
\usepackage{xcolor,xspace}
\usepackage{wrapfig}

\usepackage{graphicx}
\usepackage{bm}
\usepackage{hhline}
\usepackage{amsmath,amssymb,amsthm}
\usepackage{subcaption}
\usepackage[capitalize,noabbrev]{cleveref}

\usepackage{enumitem}
\setlist{itemsep=2pt,topsep=3pt,leftmargin=*}

\usepackage{algorithm}
\usepackage[noend]{algorithmic}

\newcommand{\set}[1]{\{#1\}}
\DeclareMathOperator*{\argmin}{argmin}
\DeclareMathOperator*{\argmax}{argmax}

\newtheorem{theorem}{Theorem}
\newtheorem{lemma}{Lemma}

\newtheorem{definition}{Definition}
\newtheorem{problem}{Problem}
\newtheorem{remark}{Remark}

\newcommand{\norm}[1]{\left\lVert#1\right\rVert}
\newcommand{\eat}[1]{}

\newcommand{\vect}[1]{\bm{#1}}
\newcommand{\oracle}{$\textsc{Marginal-Oracle}$\xspace}

\newcommand{\x}{\mathbf{x}}

\newcommand{\p}{\mathbf{p}}
\newcommand{\y}{\mathbf{y}}
\newcommand{\z}{\mathbf{z}}

\newcommand{\X}{\mathbf{X}}
\newcommand{\C}{\mathcal{C}}
\newcommand{\M}{\mathcal{M}}
\newcommand{\V}{\mathcal{V}}
\newcommand{\E}{\mathcal{E}}
\renewcommand{\L}{\mathcal{L}}
\renewcommand{\S}{\mathcal{S}}

\newcommand{\bmu}{{\boldsymbol{\mu}}}

\newcommand{\btheta}{{\boldsymbol{\theta}}}
\newcommand{\btau}{{\boldsymbol{\tau}}}

\newcommand{\I}{\mathbb{I}}
\newcommand{\R}{\mathbb{R}}

\newcommand{\grad}{\nabla}

\newcommand{\apgm}{\textsc{APPGM}\xspace}
\newcommand{\appgm}{\textsc{APPGM}\xspace} 
\newcommand{\pgm}{\textsc{Private-PGM}\xspace}
\newcommand{\prox}{\textsc{Prox-PGM}\xspace}
\newcommand{\gbp}{\textsc{Convex-GBP}\xspace}

\newcommand{\dom}{\Omega}

\def\alg{{\mathcal A}}
\def\db{\mathbf{X}}
\def\nbrs{\textrm{nbrs}}

\def\counting{\kappa}

\newcommand{\mysf}[1]{\textsf{\small{#1}}}

\title{Relaxed Marginal Consistency for Differentially Private Query Answering}
\author{%
  Ryan McKenna, Siddhant Pradhan, Daniel Sheldon, Gerome Miklau \\
  College of Information and Computer Sciences\\
  University of Massachusetts \\
  Amherst, MA 01002 \\
  \texttt{\{rmckenna, sspradhan, sheldon, miklau\}@cs.umass.edu} \\
}

\begin{document}

\maketitle

\begin{abstract}
Many differentially private algorithms for answering database queries involve a step that reconstructs a discrete data distribution from noisy measurements. This provides consistent query answers and reduces error, but often requires space that grows exponentially with dimension. \pgm is a recent approach that uses graphical models to represent the data distribution, with complexity proportional to that of exact marginal inference in a graphical model with structure determined by the co-occurrence of variables in the noisy measurements. \pgm is highly scalable for sparse measurements, but may fail to run in high dimensions with dense measurements. We overcome the main scalability limitation of \pgm through a principled approach that relaxes consistency constraints in the estimation objective. Our new approach works with many existing private query answering algorithms and improves scalability or accuracy with no privacy cost.

\end{abstract}

\section{Introduction} \label{sec:intro}

A central problem in the design of differentially private algorithms is answering sets of counting queries from a database. Many proposed algorithms follow the select-measure-reconstruct paradigm: they {\em select} a set of measurement queries, they privately {\em measure} them (using Gaussian or Laplace noise addition), and then they {\em reconstruct} the data or query answers from the noisy measurements. When done in a principled manner, the reconstruct phase serves a number of critical functions: it combines the noisy evidence provided by the measurement queries, it allows new unmeasured queries to be answered (with no additional privacy cost), and it resolves inconsistencies in the noisy measurements to produce consistent estimates, which often have lower error.  In this paper, we propose a novel, scalable, and general-purpose approach to the reconstruct step. With a principled approach to this problem, future research can focus on the challenging open problem of query selection.
%With a principled approach to this problem, future research can focus on the challenging and open problem of query selection.

Most existing \emph{general-purpose} methods for reconstruction cannot scale to high-dimensional data, as they operate over a vectorized representation of the data, whose size is exponential in the dimensionality \cite{li2010optimizing,lee2015maximum,Zhang18Ektelo,nikolov2013geometry,hardt2012simple,barak2007privacy}.  Some special purpose methods exist that have better scalability, but are only applicable within a particular mechanism or in certain special cases \cite{hay2010boosting,ding2011differentially,qardaji2014priview,zhang2021privsyn,zhang2018calm,
hardt2012simple,dwork2015efficient,liu2021leveraging,aydore2021differentially,liu2021iterative}.
%  (equivalently, a full contingency table representation), and hence fail to scale to high-dimensional data \ds{cites}. 
%
% GM: I removed this to shorten; what exactly would we have cited here?
%Recently, new techniques have been proposed that avoid the vector representation of the data, by imposing certain assumptions \ds{be more specific} on the privacy mechanism \ds{cites}. 
%
A recently-proposed method, \pgm~\cite{mckenna2019graphical}, offers the scalability of these special purpose methods and retains much of the generality of the general-purpose methods.  
\pgm can be used for the reconstruction phase whenever the measurements only depend on the data through its low-dimensional marginals.  \pgm avoids the data vector representation in favor of a more compact graphical model representation, and was shown to dramatically improve the scalability of a number of popular mechanisms while also improving accuracy~\cite{mckenna2019graphical}. \pgm was used in the winning entry of the 2018 NIST differential privacy synthetic data contest~\cite{nist,mckenna2021winning}, as well as in \emph{both} the first and second-place entry of the follow-up 2020 NIST differential privacy temporal map contest \cite{temporalmap,caidata}.
%Although proposed very recently, it was adopted by more than one of the finalists in the recent NIST differential privacy contest \cite{?}.

While \pgm is far more scalable than operating over a vector representation of the data, it is still limited. In particular, its required memory and runtime depend on the structure of the underlying graphical model, which in turn is determined by which marginals the mechanism depends on.  When the mechanism depends on a modest number of carefully chosen marginals, \pgm is extremely efficient.  But, as the number of required marginals increases, the underlying graphical model becomes intractably large, and \pgm eventually fails to run. This is due to the inherent hardness of exact marginal inference in a graphical model.

%The underlying problem is the inherent computational hardness of enforcing global consistency constraints among the implied marginal estimates.
% which is not indicative of incorrect of implementation chocies or design flaws of \pgm, but instead of inherent computationally hard problems involvings graphical models.

In this paper, we overcome the scalability limitations of \pgm by proposing a natural relaxation of the estimation objective that enforces specified \emph{local} consistency constraints among marginals, instead of global ones, and can be solved efficiently. Our technical contributions may be of broader interest.
We develop an efficient algorithm to solve a generic convex optimization problem over the local polytope of a graphical model, which uses a body of prior work on generalized belief propagation~\cite{heskes2004uniqueness,wainwright2003tree,wiegerinck2005approximations,loh2014concavity,hazan2012tightening,yedidia2005constructing,hazan2012convergent,meltzer2012convergent,heskes2006convexity,heskes2003generalized,pakzad2005estimation} and can scale to problems with millions of optimization variables.
%by exploiting structure in the constraint set, and leverages a body of prior work on generalized belief propagation algorithms for approximate free energy minimization %\cite{heskes2004uniqueness,wainwright2003tree,wiegerinck2005approximations,loh2014concavity,hazan2012tightening,yedidia2005constructing,hazan2012convergent,meltzer2012convergent,heskes2006convexity,heskes2003generalized,pakzad2005estimation}.
We also propose a variational approach to predict ``out-of-model'' marginals given estimated pseudo-marginals, which gives a completely variational formulation for both estimation and inference: the results are invariant to optimization details, including the approximate inference methods used as subroutines. 

Our new approach, \textsc{Approx-Private-PGM} (\apgm), offers many of the same benefits as \pgm, but can be deployed in far more settings, allowing effective reconstruction to be performed without imposing strict constraints on the selected measurements. We show that \apgm permits efficient reconstruction for HDMM~\cite{mckenna2018optimizing}, while also improving its accuracy, allows MWEM~\cite{hardt2012simple} to scale to far more measurements, and improves the accuracy of FEM \cite{vietri2020new}.

\section{Background}

We first review background on our data model, marginals, and differential privacy, following~\cite{mckenna2019graphical}.

\paragraph{Data}
Our input data represents a population of individuals, each contributing a single record $\x = (x_1, \ldots, x_d)$ where $x_i$ is the $i^{th}$ attribute belonging to a discrete finite domain $\dom_i$ of $n_i$ possible values.  The full domain is $\dom = \prod_{i=1}^d \dom_i$ and its size $n = \prod_{i=1}^d n_i$ is exponential in the number of attributes. A dataset $\mathbf{X}$ consists of $m$ such records $\mathbf{X} = (\x^{(1)}, \ldots, \x^{(m)})$. 
It is often convenient to work with an alternate representation of $\X$: the \emph{data vector} or \emph{data distribution} $\p$ is a vector of length $n$, indexed by $\x \in \dom$ such that $\p(\x)$ counts the fraction of individuals with record equal to $\x$. That is, $\p(\x) = \frac{1}{m} \sum_{i=1}^m \I\{\x^{(i)}  = \x\}, \forall \x \in \dom$, where $\I\{ \cdot \}$ is an indicator function. %Note that the vector $m\p$ is sometimes called the full contingency table for the data set.

\paragraph{Marginals}
When dealing with high-dimensional data, it is common to work with {\em marginals} defined over a subset of attributes.  Let $r \subseteq [d]$ be a \emph{region} or \emph{clique}
%\footnote{We use these terms interchangeably.}
that identifies a subset of attributes and, for $\x \in \dom$, let $\x_r = (x_i)_{i \in r}$ be the sub-vector of $\x$ restricted to $r$. Then the marginal vector (or simply ``marginal on $r$'') $\bmu_r$, is defined by:
\begin{equation}
\label{eq:marginal}
\bmu_r(\x_r) = \frac{1}{m}\sum_{i=1}^m \I\{ \x^{(i)}_r = \x_r\}, \quad \forall \x_r \in \dom_r := \prod_{i \in r} \dom_i.
\end{equation}
This marginal is the data vector on the sub-domain $\Omega_r$ corresponding to the attribute set~$r$. Its size is  $n_r := |\dom_r| = \prod_{i \in r} n_i$, which is exponential in $|r|$ but may be considerably smaller than $n$. A marginal on $r$ can be computed from the full data vector or the marginal for any superset of attributes by summing over variables that are not in $r$.
%, which is a linear operation.
We denote these (linear) operations by $M_r$ and $P_{s\to r}$, so $\bmu_r = M_r \p = P_{s\to r}\bmu_s$ for any $r \subseteq s$. 
We will also consider vectors $\bmu$ that combine marginals for each region in a collection $\C$, and let $M_\C$ be the linear operator such that $\bmu = (\bmu_r)_{r \in \C} = M_{\C} \p$.

\paragraph{Differential Privacy}
Differential privacy protects individuals by bounding the impact any one individual can have on the output of an algorithm. 

\begin{definition}[Differential Privacy \cite{Dwork06Calibrating}] \label{def:dp}
A randomized algorithm $\alg$ satisfies $(\epsilon, \delta)$-differential privacy if, for any input $\db$, any $\db' \in \nbrs(\db)$, and any subset of outputs $S \subseteq \textrm{Range}(\alg)$,
$$ \Pr[\alg(\db) \in S] \leq \exp(\epsilon) \Pr[\alg(\db') \in S] + \delta$$
\end{definition}

Above, $\nbrs(\db)$ denotes the set of datasets formed by replacing any $\x^{(i)} \in \db$ with an arbitrary new record $\x'^{(i)} \in \dom$.  When $ \delta = 0 $ we say $ \alg $ satisfies $\epsilon$-differential privacy.

\section{Private-PGM}

In this section we describe \pgm \cite{mckenna2019graphical}, a general-purpose reconstruction method applied to differentially private measurements of a discrete dataset. There are two steps to \pgm: (1) \emph{estimate} a representation of the data distribution given noisy measurements, and (2) \emph{infer} answers to new queries given the data distribution representation.

In particular, suppose an arbitrary $(\epsilon, \delta)$-differentially private algorithm $\alg$ is run on a discrete dataset with data vector $\p_0$, where $\alg$ only depends on $\p_0$ through its low-dimensional marginals $\bmu_0 = M_\C \p_0$ for some collection of cliques $\C$.  The sample $\y \sim \alg(\bmu_0)$ typically reveals noisy high-level aggregate information about the data. \pgm will first estimate a compact representation of a distribution $\hat\p$ that explains $\y$ well, and then answer new queries using $\hat\p$.

\paragraph{Estimation: Finding a Data Distribution Representation} \pgm first estimates a data vector by finding $\hat\p$ to solve the inverse problem $\min_{\p} L(M_\C \p)$, where $L(\bmu)$ is a convex loss function that measures how well $\bmu$ explains the observations $\y$.
Since $L(M_\C \p)$ only depends on $\p$ through its marginals, it is clear we can find the optimal \emph{marginals} by instead solving the following problem.

\begin{problem}[Convex Optimization over the Marginal Polytope] \label{prob:marginal}
Given a clique set $\C$ and convex loss function $L(\bmu)$, solve
$$
\hat{\bmu} \in \argmin_{\bmu \in \M(\C)} L(\bmu),
$$
where $\M(\C) = \{ \bmu: \exists\; \p \text{ s.t. } M_\C \p = \bmu \}$ is the set of realizable marginals, known as the \emph{marginal polytope} of $\C$ \cite{wainwright2008graphical}.
\end{problem}
The solution to this problem gives marginals that are \emph{consistent} with some underlying data vector and, therefore, typically provide a better estimate of the true marginals than~$\y$. 
In the general case, the loss function $L$ can simply be set to the \emph{negative log likelihood}, i.e., $L(\bmu) = -\log \Pr[\alg(\bmu) = \y]$,\footnote{For mechanisms with continuous output values, interpret this as a negative log-density.} however other choices are also possible.
As a concrete motivating application, consider the case where the mechanism $\alg$ adds Gaussian noise directly to the data marginals $\bmu_0$, i.e., $\alg(\bmu_0) = \y$ where $\y_r = \bmu_{0,r} + \mathcal{N}(0, \sigma^2 I_{n_r})$. In this case, the log-likelihood is proportional to the squared Euclidean distance and gives the loss function $ L(\bmu) = \norm{\bmu - \y}_2^2$, so the problem at hand is an $L_2$ minimization problem. The theory for \pgm  focuses on convex loss functions, but the algorithms are also used to seek local minima of \cref{prob:marginal} when $L$ is non-convex.

\begin{wrapfigure}{R}{0.42\textwidth}
\vspace{-2.5em}
\begin{minipage}{0.42\textwidth}
\begin{algorithm}[H]
    \caption{\prox \cite{mckenna2019graphical}} \label{alg:proximal}
\begin{algorithmic}
    \STATE {\bfseries Input:} Convex loss function $L(\bmu)$
    \STATE {\bfseries Output:} Marginals $\hat\bmu$, parameters $\hat\btheta$ %Estimated data distribution $\hat{\p}_{\btheta}$
    \STATE $\hat\btheta = \vect{0}$
    \FOR{$t=1, \dots, T$}
        \STATE $\hat\bmu = \oracle(\hat\btheta)$
        \STATE $\hat\btheta = \hat\btheta - \eta_t \grad L(\hat\bmu)$
    \ENDFOR
    \STATE {\bfseries return} $\hat\bmu$, $\hat\btheta$ %, $\hat{\p}_{\btheta}$
\end{algorithmic}
\end{algorithm}
\end{minipage}
\vspace{-2.5em}
\end{wrapfigure}
\paragraph{Graphical models} Two remaining issues are how to solve \cref{prob:marginal} and how to recover a full data vector from $\hat\bmu$. \pgm addresses both with \emph{graphical models}.
A graphical model with clique set $\C$ is a distribution over $\Omega$ where the unnormalized probability is a product of factors involving only subsets of variables, one for each clique in $\C$. It has the form
$$
\p_\btheta(\x) = \frac{1}{Z} \exp\left( \sum_{r \in \C} \btheta_r(\x_r) \right).
$$
The real numbers $\btheta_r(\x_r)$ are \emph{log-potentials} or \emph{parameters}. The full parameter vector $\btheta = (\btheta_r(\x_r))_{r \in \C,\x_r \in \Omega_r}$ matches the marginal vector $\bmu$ in size and indexing, and the relationship between these two vectors is central to graphical models~\cite{wainwright2008graphical}:
\begin{itemize}
\item A parameter vector $\btheta$ determines a unique marginal vector $\bmu_{\btheta} \in \M(\C)$, defined as $\bmu_\btheta = M_\C\p_\btheta$, the marginals of $\p_\btheta$. \emph{Marginal inference} is the problem of (efficiently) computing $\bmu_\btheta$ from $\btheta$. It can be solved exactly by algorithms such as \emph{variable elimination} or \emph{belief propagation} with a junction tree~\cite{koller2009probabilistic}. We denote by \oracle an algorithm that outputs $\bmu_\btheta$ on input~$\btheta$.
\item For every $\bmu \in \M(\C)$ with positive entries, there is a unique distribution $\p_\btheta$ in the family of graphical models with cliques $\C$ that has marginals~$\bmu$, and $\p_\btheta$ has \emph{maximum entropy among all distributions with marginals $\bmu$}.  
%We use this principle to recover a full distribution from estimated marginals.
\end{itemize}

\prox (Algorithm~\ref{alg:proximal}) is a proximal algorithm to solve \cref{prob:marginal}~\cite{mckenna2019graphical}. It returns a marginal vector $\hat\bmu \in \M(\C)$ that minimizes $L(\bmu)$ \emph{and} parameters $\hat\btheta$ such that $\p_{\hat\btheta}$ has marginals $\hat\bmu$. The core of the computation is repeated calls to \oracle. The estimated graphical model $\p_{\hat\btheta}$ has cliques $\C$ that coincide with the marginals measured by the privacy mechanism, and has maximum entropy among all distributions whose marginals minimize $L(\bmu)$.
In general, there will be infinitely many distributions that have marginals $\hat\bmu$ (and hence are equally good from the perspective of the loss function $L$). \pgm chooses the distribution with maximum entropy, which is an appealing way to break ties that falls out naturally from the graphical model. %\ry{think we should say more here} \ds{Seems OK to me; changed order of paragraphs}

\paragraph{Inference: Answering New Queries}
With $\p_{\hat\btheta}$ in hand, \pgm can readily estimate new marginals $\bmu_r$.  There are two separate cases. If $r$ is contained in a clique of  $C$, we say $r$ is ``in-model'', and we can readily calculate $\bmu_r$ from the output of \prox. The more interesting case occurs when $r$ is out-of-model (is not contained in a clique of $\C$): in this case, the standard way to compute $\bmu_r$ is to perform variable elimination in the graphical model $\p_{\btheta}$. %Variable elimination can be seen as running marginal inference in an expanded graphical model with cliques $\C \cup \{ r \} $ and parameters $\btheta_r = \mathbf{0}$ for the new clique.  Clearly, introducing this new clique with zero log-potential does not change the distribution $\p_{\btheta}$.  

\begin{remark}The complexity of \pgm depends on that of \oracle, which depends critically on the structure of the cliques $\C$ measured by the privacy mechanism. In general, running time is exponential in the treewidth of the graph $G$ induced by attribute co-occurrence within a clique of $\C$. When $G$ is tree-like, \pgm can be highly efficient and exponentially faster than working with a full data vector. When $G$ is dense, \pgm may fail to run due to time or memory constraints. This limitation is not specific to the \prox algorithm: \cref{prob:marginal} is as hard as marginal inference, which can be solved by minimizing the (convex) variational free energy over the marginal polytope~\cite{koller2009probabilistic}. A primary difficulty is the intractability of $\M(\C)$, which is a convex set, but in general requires a very large number of constraints to represent explicitly~\cite{wainwright2008graphical}.  In two state-of-the-art mechanisms for synthetic data, MST \cite{mckenna2021winning} and PrivMRF \cite{caidata}, $\C$ was specifically chosen to limit the treewidth and ensure tractability of Private-PGM.  In other mechanisms agnostic to the limitations of \pgm, like HDMM and MWEM, the set $\C$ can often lead to graphs with intractable treewidths.
\end{remark}

\section{Our Approach}\label{sec:approach}

In this section we describe our approach to overcome the main scalability limitations of \pgm, by introducing suitable approximations and new algorithmic techniques.  Our innovations allow us to scale significantly better than \pgm with respect to the size of $\C$. A high-level idea is to use approximate marginal inference in the \prox algorithm, but doing so naively would make it unclear what, if any, formal problem is being solved. We will develop a principled approach that uses approximate inference to \emph{exactly} solve a relaxed optimization problem.

\begin{figure}[t]
  \begin{subfigure}[b]{0.18\textwidth}
    \centering\includegraphics[width=\textwidth]{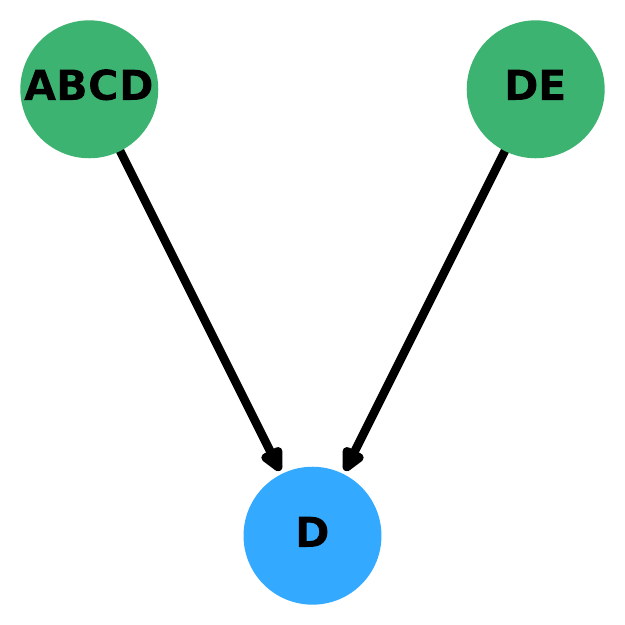}
    \caption{\label{fig:junction} Junction Tree}
  \end{subfigure}%
  \hspace{2em}
  \begin{subfigure}[b]{0.36\textwidth}
    \centering\includegraphics[width=\textwidth]{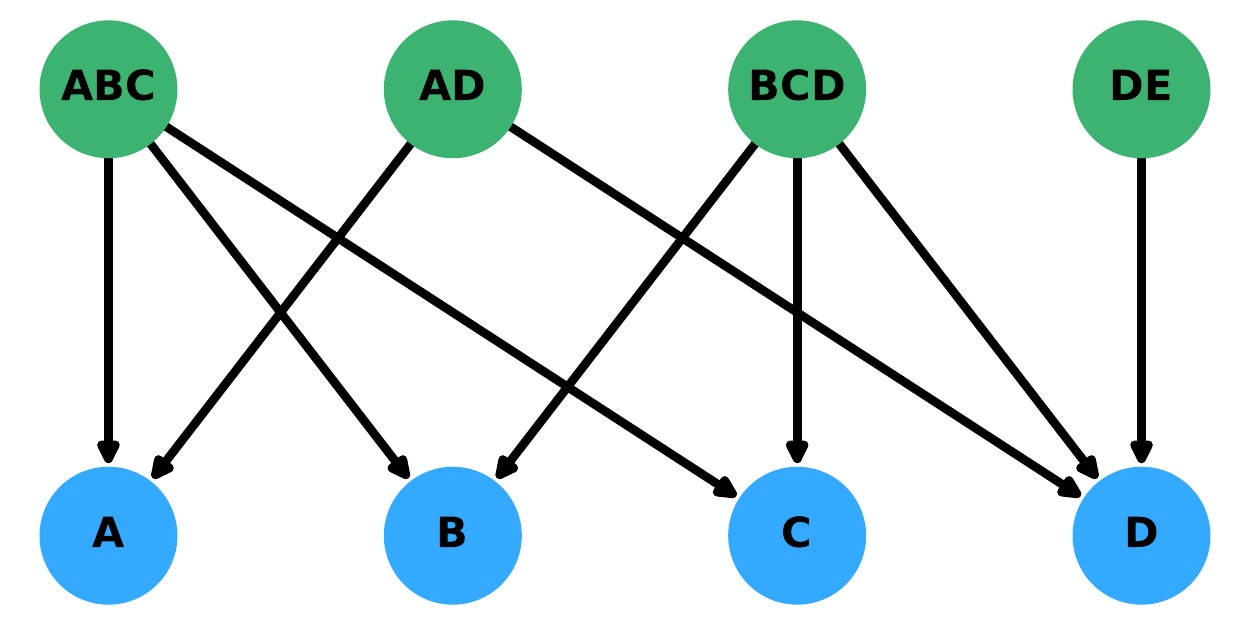}
    \caption{\label{fig:factor} Factor Graph}
  \end{subfigure}%
  \hspace{2em}
  \begin{subfigure}[b]{0.36\textwidth}
    \centering\includegraphics[width=\textwidth]{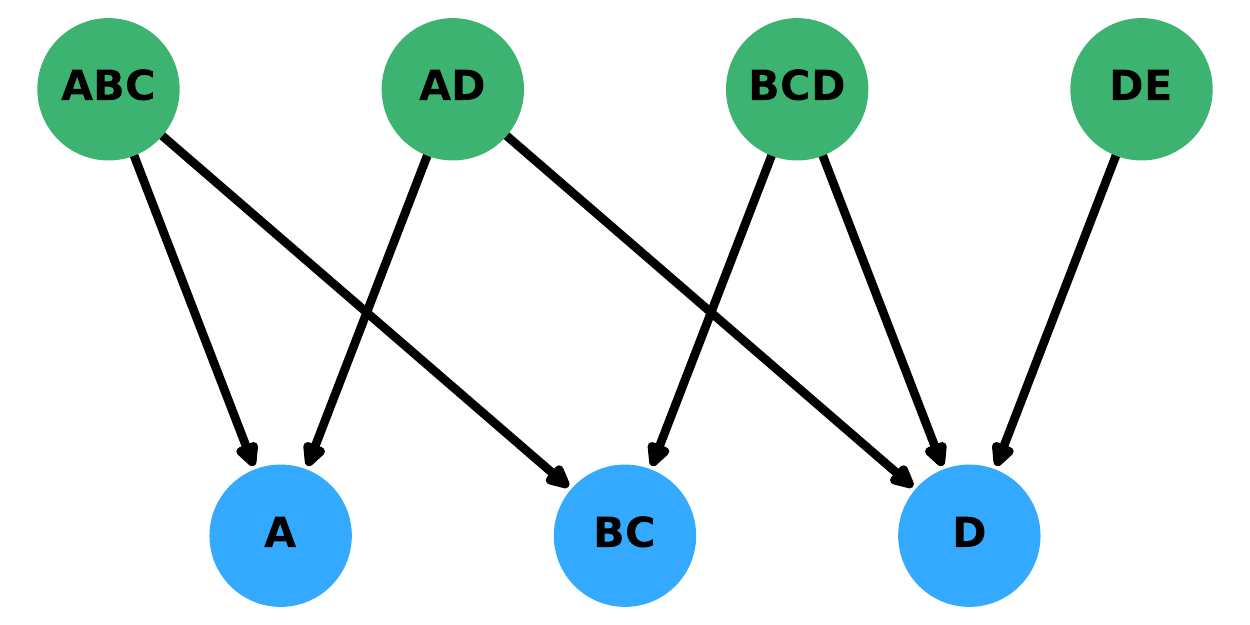}
    \caption{\label{fig:region} Saturated Region Graph}
  \end{subfigure}%
\caption{\label{fig:regions} Comparison of different region graph structures defined over a domain with attributes $\set{A,B,C,D,E}$ that support the cliques $\C = \set{\set{A,B,C}, \set{A,D}, \set{B,C,D}, \set{D,E}}$.}
\end{figure}

\paragraph{Region Graphs}
Central to our approach is the notion of a \emph{region graph}, which is a data structure that encodes constraints between cliques in a natural graphical format and facilitates message passing algorithms for approximate marginal inference.

\begin{definition}[Region Graph \cite{koller2009probabilistic}]
A region graph $G = (\V, \E)$ is a directed graph where every vertex $r \in \V$ is an attribute clique and for any edge $r \rightarrow s \in \E$ we have that $ r \supseteq s $. We say that $G$ supports a clique set $\C$ if for every clique $r \in \C$, there exists some $r' \in \V$ such that $r \subseteq r'$.
\end{definition}
For any region graph, there is a corresponding set of constraints that characterize the \emph{local polytope} of internally consistent \emph{pseudo-marginals}, defined below.  

\begin{definition}[Local Polytope \cite{koller2009probabilistic}]
The local polytope of pseudo-marginals associated with a region graph $G=(\V,\E)$ is:
\begin{equation} \label{eq:localpolytope}
\L(G) =
  \left\{ \btau \geq 0
  \;\middle\vert\;
  \begin{array}{@{}l@{}l@{}}
  %\sum_{\x_r} \btau_r(\x_r) = 1 & \forall r \in \V \\
  %\sum_{\x_r \setminus \x_t} \btau_r(\x_r) = \btau_t(\x_t) \quad\quad & \forall r \rightarrow t \in \E\\
  \mathbf{1}^{\top} \btau_r = 1 & \forall r \in \V \\
  P_{r \rightarrow s} \btau_r = \btau_s \quad \quad & \forall r \rightarrow s \in \E \\
  \end{array}
  \right\}.
\end{equation}
\end{definition}

The nodes in the region graph correspond to the cliques in the pseudo-marginal vector, while the edges in the region graph dictate which internal consistency constraints we expect to hold between two cliques.  These constraints are necessary, but not sufficient, for a given set of pseudo-marginals to be realizable, i.e., $\M(\V) \subseteq \L(G)$.   In the special case when $G$ is a junction tree, these constraints are also sufficient, and we have $\M(\V) = \L(G)$.We use the notation $\btau$ in place of $\bmu$ to emphasize that $\btau$ is not necessarily a valid marginal vector, even though we will generally treat it as such.  This notational choice is standard in the graphical models literature \cite{wainwright2008graphical}.  The general idea is to relax problems involving the intractable marginal polytope to use the local polytope instead, since $\L(G)$ is straightforward to characterize using the linear constraints in Equation~\eqref{eq:localpolytope}.  
%The nodes and edges in the region graph dictate which constraints we expect to hold among pseudo-marginals.  

Region graphs can encode different structures, including junction trees and factor graphs as special cases.
For example, \cref{fig:regions} shows three different region graphs that support $\C = \set{\set{A,B,C}, \set{A,D}, \set{B,C,D}, \set{D,E}}$.  At one extreme is the \mysf{Junction Tree}, shown in \cref{fig:junction}, which is obtained by merging cliques $\set{A,B,C}$, $\set{A,D}$, and $\set{B,C,D}$ into a super-clique $\set{A,B,C,D}$.  Here, $\M(\V) = \L(G)$, and ordinary belief propagation in this graph corresponds to exact marginal inference.  At the other end of the extreme is the \mysf{Factor Graph}, shown in \cref{fig:factor}.  This graph contains one vertex for every clique $r \in \C$, plus additional vertices for the singleton cliques.  It encodes constraints that all cliques must agree on common one-way marginals. For example, $\btau_{ABC}$ and $\btau_{BCD}$ must agree on the shared one-way marginals $\btau_B$ and $\btau_C$, but not necessarily on the shared two-way marginal $\btau_{BC}$.  A natural middle ground is the fully \mysf{Saturated Region Graph}, shown in \cref{fig:region}.  This graph includes every clique $r \in \C$ as a vertex, and includes additional vertices to capture intersections between those cliques.  Unlike the factor graph, this graph \emph{does} require that $\btau_{ABC}$ is consistent with $\btau_{BCD}$ with respect to the $\btau_{BC}$ marginal.  Unlike the junction tree, this graph does not require forming super-cliques whose size grow quickly with $|\C|$.  For more details about the concepts above, please refer to \cite[][Section 11.3]{koller2009probabilistic}.
%\ds{Give citations for facts and background concepts in this pargraph that are not fully explained or supported, e.g., junction tree, saturated region graph. A blanket citation might work, but targeted could be better.} \ry{agree targeted is nice, but I think I would just cite Koller and Friedman for all of these.}

The methods we describe in this paper apply for any region graph that supports $\C$.
By default, we simply use the fully saturated region graph, which is the smallest region graph that encodes all internal consistency constraints, and can easily be constructed given the cliques $\C$ \cite{koller2009probabilistic}.

\eat{\begin{definition}[Saturated Region Graph]
Given a set of cliques $\C$, the associated saturated region graph $G=(\V,\E)$ is defined as
\begin{enumerate}
\item The vertices are the closure of $\C$ under intersection (i.e., the smallest set $\V$ that contains $\C$ and is closed under intersection).  
\item Whenever $r \supseteq s $ and there is no $t$ such that $ r \supseteq t \supseteq s$, there is an edge $r \rightarrow s$.  
\end{enumerate}
\end{definition}}

%% \ds{Orphaned:} For a graphical model $\p_\btheta$ on clique set $\C$ and a region graph that supports $\C$, the goal is to estimate the marginal distributions associated with each vertex in the graph, i.e., $\btau = (\btau_r)_{r \in V}$. Since $G$ supports $\C$, we can easily compute $(\btau_r)_{r \in \C}$ from $(\btau_r)_{r \in V}$ via marginalization, so this solves our original inference goal of estimating $\M_\C \p_\btheta$. Additionally, we can optimize any function over $\M(\C)$ by optimizing an equivalent function over $M(\V)$.\footnote{Let $P$ be the marginalization operation to compute $\btau_r$ from $\btau_{r'}$ for some $r' \in \V$ for each $r \in \C$. $P$ is linear and maps $\M(\V)$ onto $\M(\C$), therefore $\min_{\btau \in \M(\C)} L(\btau) = \min_{\btau' \in \M(\V)}L'(\btau')$ for $L'(\btau') = L(P\btau')$.}

\paragraph{Estimation: Finding an Approximate Data Distribution Representation}

We begin by introducing a very natural relaxation of the problem we seek to solve.

\begin{problem}[Convex Optimization over the Local Polytope] \label{prob:convex}
Given a region graph $G$ and a convex loss function $L(\btau)$ where $\btau = (\btau_r)_{r \in \V}$, solve:
$$ \hat{\btau} = \argmin_{\btau \in \L(G)} L(\btau). $$ 
\end{problem}

In the problem above, we simply replaced the marginal polytope $\M(\V)$ from our original problem\footnote{If $G$ supports $\C$, we can assume without loss of generality that \cref{prob:marginal} was defined on $\M(\V)$ instead of $\M(\C)$. In particular, the loss function $L$ can be written to depend on marginals $(\bmu_{r'})_{r' \in \V}$ instead of $(\bmu_r)_{r \in \C}$, because the latter can be computed from the former.} with the local polytope $\L(G)$. Since this is a convex optimization problem with linear constraints, it can be solved with a number of general purpose techniques, including interior point and active set methods \cite{boyd2004convex}.  However, these methods do nothing to exploit the special structure in the constraint set $\L(G)$, and as such, they have trouble running on large-scale problems.  

Our first contribution is to show that we can solve \cref{prob:convex} efficiently by instantiating \prox with a carefully chosen approximate marginal oracle.  To do so, it is useful to view approximate marginal inference through the lens of the free energy minimization problem, stated below.

\begin{problem}[Approximate Free Energy Minimization \cite{koller2009probabilistic}] \label{prob:energy}
Let $G$ be a region graph, $\btheta = (\btheta_r)_{r \in \V}$ be real-valued parameters, and $H_{\kappa}(\btau) = \sum_{r \in \V} \kappa_r H(\btau_r) $, where $\kappa_r \in \mathbb{R}$ are counting numbers, and $H(\btau_r) = -\sum_{\x_r \in \dom_r} \btau_r(\x_r) \log{\btau_r(\x_r)}$ is the Shannon entropy of $\btau_r$, solve:
$$ \hat{\btau} = \argmin_{\btau \in \L(G)} -\btau^{\top} \btheta - H_{\kappa}(\btau) $$
\end{problem}

This problem approximates the (intractable) variational free energy minimization problem~\cite{wainwright2008graphical}, for which the optimum gives the true marginals of $\p_\btheta$, by using $\L(G)$ instead of $\M(\V)$, and using $H_{\kappa}(\btau)$ as an approximation to the full Shannon entropy.
%\ds{TODO: add sentence here to connect to exact inference} If $G$ is a Junction Tree, and the counting numbers $\counting$ are chosen appropriately, then it can be shown that the marginals of the graphical model with parameters $\btheta$ exactly solve \cref{prob:energy}.
Many algorithms for approximate marginal inference can be seen as solving variants of this free energy minimization problem under different assumptions about $G$ and $\counting$ \cite{heskes2004uniqueness,wainwright2003tree,wiegerinck2005approximations,loh2014concavity,hazan2012tightening,yedidia2005constructing,hazan2012convergent,meltzer2012convergent,heskes2006convexity,heskes2003generalized,pakzad2005estimation}.  %\ds{The rest of this paragraph has background that may not be needed. Consider dropping parts about Bethe and loopy BP especially, maybe also TRW. Add citations for statements like ``many algorithms can be seen as...''} Indeed, one of the most famous approximate inference algorithms, \emph{loopy belief propagation} on a factor graph, falls into this category.  In particular, if loopy belief propagation converges, it converges to a fixed point of \cref{prob:energy}, where $\counting$ correspond to the Bethe-entropy approximation.  Unfortunately, the Bethe entropy is not convex except in special cases, and loopy belief propagation is not guaranteed to converge in general.  For this reason, variants of loopy belief propagation have been proposed that can be shown to solve \cref{prob:energy} where $\counting$ are \emph{convex} counting numbers (i.e., they correspond to a convex entropy approximation).  One well-known algorithm from this class is tree-reweighted belief propagation \cite{wainwright2003tree}.  However, this algorithm is only applicable for pairwise models (i.e., $|r| \leq 2$ for all $ r \in V$).  A more general but lesser-known algorithm was proposed by Hazan et al. \cite{hazan2012tightening} and is shown in \cref{alg:mpconvex}.  Other related algorithms for this problem also exist \cite{wiegerinck2005approximations,loh2014concavity}. 

\begin{theorem}[Algorithm for Approximate Free Energy Minimization \cite{hazan2012tightening}] \label{thm:energy}
Given a region graph $G$, parameters $\btheta = (\btheta_r)_{r \in \V}$, and any positive counting numbers $\counting_r > 0$ for $ r \in \V$, the convex generalized belief propagation (\gbp) algorithm of~\cite{hazan2012tightening} solves the approximate free energy minimization problem of \cref{prob:energy}.  
\end{theorem}

\gbp is listed in \cref{sec:approxalg} (Algorithm~\ref{alg:mpconvex}) and is a message-passing algorithm in the region graph that uses the counting numbers as weights. 
Importantly, the complexity of \gbp depends mainly on the size of the largest clique in the region graph.  In many cases of practical interest, this will be exponentially smaller in the saturated region graph than in a junction tree. %\ds{Suggest jumping straight from here to Theorem.} %One limitation of \cref{alg:mpconvex} is that its correctness depends on every counting number being positive.  This is somewhat unnatural, because when $G$ is a junction tree, exact marginal inference solves \cref{prob:energy} with positive and negative counting numbers.  Hence, \cref{alg:mpconvex} does not reduce to ordinary belief propagation when $G$ is a junction tree.  Remarkably, this drawback will not matter for our purposes, as we show in \cref{thm:approx}  

\begin{theorem} \label{thm:approx}
When \prox uses \gbp as the \oracle (with \textbf{any} positive counting numbers $\kappa$), it solves the convex optimization problem over the local polytope of \cref{prob:convex}.
\end{theorem}

This result is remarkable in light of previous work, where different counting number schemes are used with the goal of tightly approximating the true entropy, and form the basis for different approximate inference methods. In our setting, all methods with \emph{convex} counting numbers are equivalent: they may lead to different \emph{parameters} $\hat\btheta$, but the corresponding pseudo-marginals $\hat\btau = \gbp(\hat\btheta)$ are invariant. Indeed, the optimal $\hat\btau$ depends only on the estimation objective $L(\btau)$ and the structure of the local polytope. We conjecture that a similar invariance holds for traditional marginal-based learning objectives with approximate inference~\cite{domke2013learning} when message-passing algorithms based on convex free-energy approximations are used as the approximate inference method.

\begin{proof}
Since $L$ is a convex function and $\L$ is a convex constraint set, this problem can be solved with mirror descent \cite{beck2003mirror}.  Each iteration of mirror descent requires solving subproblems of the form:
$$\btau^{t+1} = \argmin_{\btau \in \L} \btau^{\top} \grad L(\btau^t) + \frac{1}{\eta_t} D(\btau, \btau^t), \quad\quad D(\btau, \btau^t) = \psi(\btau) - \psi(\btau^t) - (\btau - \btau^t)^{\top} \grad \psi(\btau^t).$$
Here, $D$ is a Bregman distance measure and $\psi$ is some strongly convex and continuously differentiable function.  Setting $ \psi = -H_{\counting}$, a negative weighted entropy with any (strongly) convex counting numbers $\counting$, we arrive at the following update equation:
\begin{align*}
%\btau^{t+1} &= \argmin_{\btau \in \L} \btau^{\top} \grad L(\btau^t) + \frac{1}{\eta_t} D(\btau, \btau^t) \\
\btau^{t+1} &= \argmin_{\btau \in \L(G)} \btau^{\top} \grad L(\btau^t) + \frac{1}{\eta_t} \Big( -H_{\counting}(\btau) + H_{\counting}(\btau^t) + (\btau - \btau^t)^{\top} \grad H_{\counting}(\btau^t) \Big) \\
%&= \argmin_{\btau \in \L} \btau^{\top} \grad L(\btau^t) + \frac{1}{\eta_t} \Big( -H_{\rho}(\btau) + \btau^T \grad H_{\rho}(\btau^t) \Big) \\
%&= \argmin_{\btau \in \L} \btau^{\top} \Big(\grad L(\btau^t) + \frac{1}{\eta_t} \grad H_{\rho}(\btau^t) \Big) - \frac{1}{\eta_t} H_{\rho}(\btau) \\
&= \argmin_{\btau \in \L(G)} \btau^{\top} \Big(\eta_t \grad L(\btau^t) + \grad H_{\counting}(\btau^t) \Big) - H_{\counting}(\btau) \tag{algebraic manipulation} \\
&= \argmin_{\btau \in \L(G)} \btau^{\top} \Big(\eta_t \grad L(\btau^t) - \btheta^t \Big) - H_{\counting}(\btau) \tag{\cref{lem:cvxdual}; \cref{sec:approxalg}} \\
%&= \btau \big( -\eta_t \grad L(\btau^t) - \grad H(\btau^t) \big) \\
&= \gbp \big(G,\, \btheta^t -\eta_t \grad L(\btau^t),\, \counting \big) \tag{\cref{thm:energy}}
\end{align*}
\end{proof}

\paragraph{Inference: Answering New Queries} 
We now turn our attention to the problem of inference.  The central challenge is to estimate out-of-model marginals. Let $\hat\btau$ and $\hat\btheta$ be the estimated pseudo-marginals and corresponding parameters after running \prox with \gbp and region graph $G$. We have $\hat\btau \approx \bmu_0$, and want an estimate $\hat\btau_r \approx \bmu_{0, r}$ where $r \not \in \V$. 

\gbp is the mapping such that $\hat\btau = \gbp(\hat\btheta) \approx \bmu_0$. Thus, it is appropriate to use \gbp with estimated parameters $\hat\btheta$ as the basis for estimating new pseudo-marginals. This requires selecting an expanded region graph $G'$ that supports $r$ and new counting number $\kappa_r$.
%The new parameters can be set to $\btheta_r = \textbf{0}$.
In \cref{sec:extraout}, we analyze this approach for an idealized setting and find that it leads to estimates $\hat\btau_r$ that maximize the entropy $H(\hat\btau_r)$ subject to $\hat\btau_r$ being consistent with $\hat\btau$ on overlapping marginals. However, there are two practical difficulties with the idealized setting. First, there may be \emph{no} $\hat\btau_r$ that is consistent with $\hat\btau$ on overlapping marginals: this is because $\hat\btau$ satisfies only local consistency constraints. Second, the idealized case uses $\kappa_r$ very close to zero, and \gbp performs poorly in this case. Instead, we design an optimization algorithm to mimic the idealized setting:

\begin{problem}[Maximize Entropy Subject to Minimizing Constraint Violation] \label{prob:project}
Let $\hat\btau = (\hat\btau_u)_{u \in \V}$ and let $r \notin \V$. Solve:
$$
\max_{\hat\btau_r} H(\hat\btau_r) \text{ subject to } \hat\btau_r \in \argmin_{\btau_r \in \S} \sum_{u \in \V, s=u\cap r}\|P_{r \to s}\btau_r - P_{u \to s}\hat\btau_u \|_2^2.
$$
\end{problem}
This relaxes the constraint that $\hat\btau_r$ agrees with $\hat\btau$ on overlapping marginals, to instead minimize the $L_2$ constraint violation. The inner problem is a quadratic minimization problem over the probability simplex $\S$. We show in \cref{sec:extraout} that a maximizer of \cref{prob:project} is obtained by solving the inner problem once using entropic mirror descent~\cite{beck2003mirror}.

The advantages of \cref{prob:project} are that it is low-dimensional, only requires information from $\hat\btau$ about attributes that are in $r$, can be solved much more quickly than running \gbp, and can be solved in parallel for different marginals $r, r'$. This also gives a \emph{fully} variational approach: both estimation and inference are fully defined through convex optimization problems that can be solved efficiently, and whose solutions are invariant to details of the approximate inference routines such as counting numbers.

%\paragraph{Alternate approach: \prox with Expanded Region Graph}
%Another simple option is to expand the region graph to include all desired cliques \emph{prior} to running \prox. That is, we define a region graph $G'$ at the outset with region set $\V'$ that contains each marginal $r$ we want to estimate, and use $\V'$ when running \prox. Then we obtain an $\hat\btau_r$ with no additional work after running \prox. The downside of this approach is that when the number of new marginals to estimate is large compared to $\V$, it can increase the running time of \prox by a large amount. 

\section{Experiments} \label{sec:experiments}

\paragraph{Comparison to \pgm in a Simple Mechanism}

We begin by comparing the accuracy and scalability of \apgm and \pgm for estimating a fixed workload of marginal queries from noisy measurements of those marginals made by the Laplace mechanism.
We use synthetic data to control the data domain and distribution (details in \cref{sec:synthetic}) and measure $k$ different 3-way marginals with $\epsilon=1$, for different settings of $k$.
We run \prox with different versions of \oracle: \mysf{Exact}, \mysf{Saturated Region Graph}, and \mysf{Factor Graph}, where the first corresponds to \pgm, and the latter two to \apgm  with the corresponding region graph (\cref{fig:regions}).

%The first is an exact marginal oracle defined over a \textbf{Junction Tree}, while the latter two are approximate marginal oracles defind over a , respectively.

\textit{Accuracy.} We first show that when exact inference is tractable, some accuracy is lost by using approximate inference, but the estimated pseudo-marginals are much better than the noisy ones.
We use eight-dimensional data with $n_1 = \dots = n_8 = 4$, which is small enough so exact inference is always tractable, and measure random 3-way marginals for $k$ from $1$ to $\binom{8}{3}=56$.
%On this domain, even when every possible clique is included, the total size of the junction tree is small enough to allow for exact marginal inference.  %In addition to the junction tree, we consider three different region graph structures: the over-saturated region graph (described in \cref{sec:inference}), the saturated region graph, and the factor graph (described in \cref{sec:region}). 
%We fix the privacy budget at $\epsilon=1.0$ and measure $k$ random 3-way marginals using the Laplace mechanism, where $k$ is varied from $1$ to $\binom{8}{3}=56$.
We then run $10000$ iterations of \prox using differing choices for \oracle, 
%denoted as follows: (1) {\sc EXACT}(Junction tree), an exact oracle on a junction tree, (2) {\sc APPROX}(OSat region graph), an approximate oracle on an over-saturated region graph, (3) {\sc APPROX}(Sat region graph)	, an approximate oracle using a saturated region graph, and (4) {\sc APPROX}(factor-graph) an approximate oracle using a factor graph.  
and report the $L_1$ error on in- and out-of-model marginals, averaged over all cliques and across five trials.

%in \cref{fig:inclique} and out-of-model marginals in \cref{fig:outclique}, 

%\gm{Can we use this notation in Fig 2 and maybe the text above? I think figures would be easier to read.}
%\begin{itemize}
%\item {\sc EXACT}(Junction tree)	
%\item {\sc APPROX}(OSat region graph)
%\item {\sc APPROX}(Sat region graph)	
%\item {\sc APPROX}(Factor graph)	
%\end{itemize}

In \cref{fig:inclique}, we see that the error of all \prox variants is always lower than the error of the noisy measurements themselves. \mysf{Exact} (\pgm) always has lowest error, followed by \mysf{Saturated Region Graph}, then \mysf{Factor Graph}. This matches expectations: richer region graph structures encode more of the actual constraints and hence provide better error. %\ry{Do we want to say more here about the tradeoffs?}
The trend is similar for out-of-model marginals (\cref{fig:outclique}). \mysf{Factor Graph}, which only enforces consistency with respect to the one-way marginals, performs poorly on unmeasured cliques, while \mysf{Saturated Region Graph} and \mysf{Exact}, which enforce more constraints, do substantially better. The difference between \mysf{Saturated Region Graph} and \mysf{Exact} is smaller, but meaningful.

\begin{figure}
  \centering
  \begin{subfigure}[b]{0.33\textwidth}
    \centering\includegraphics[width=\textwidth]{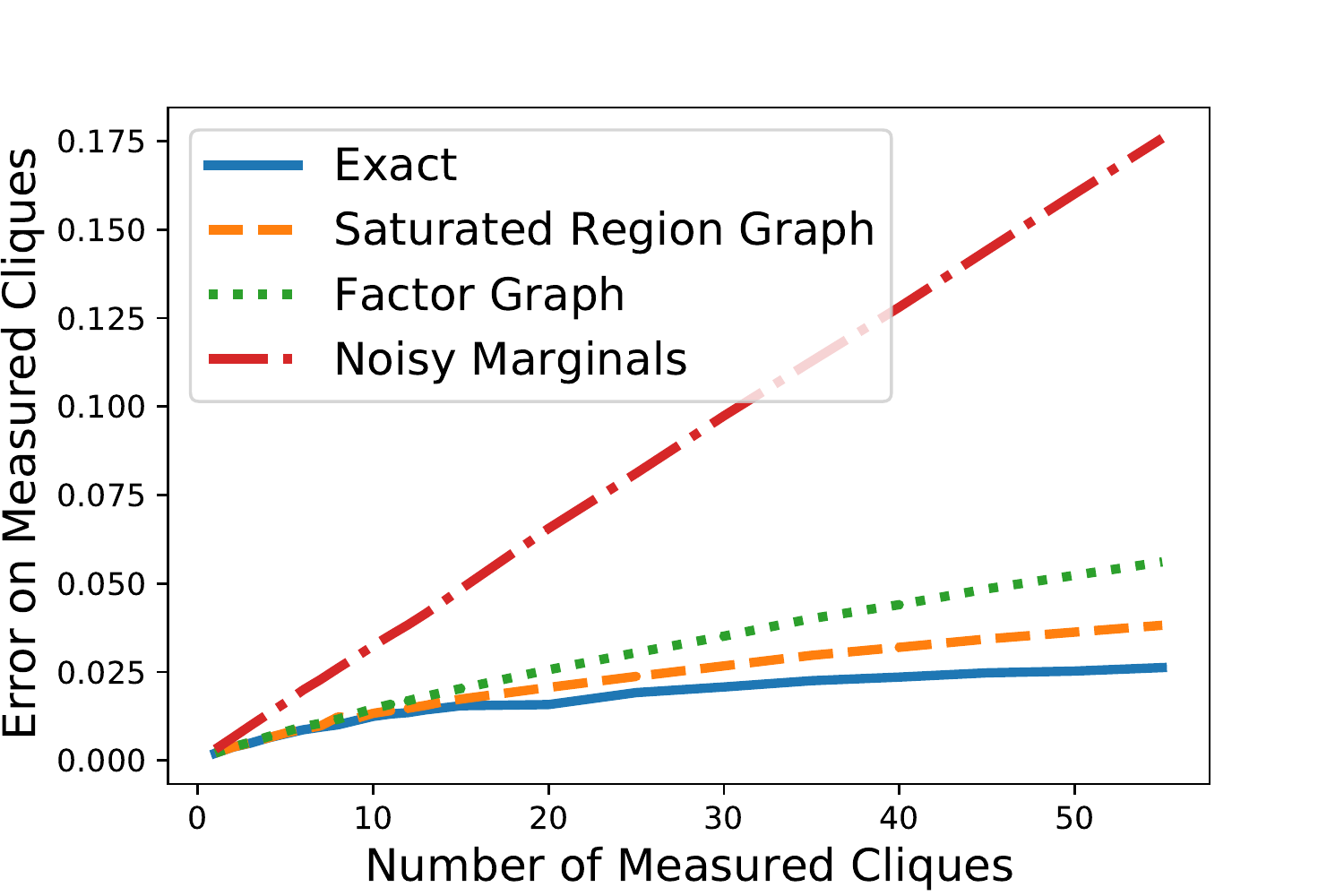}
    \caption{\label{fig:inclique} In-Model Marginals}
  \end{subfigure}%
  \begin{subfigure}[b]{0.33\textwidth}
    \centering\includegraphics[width=\textwidth]{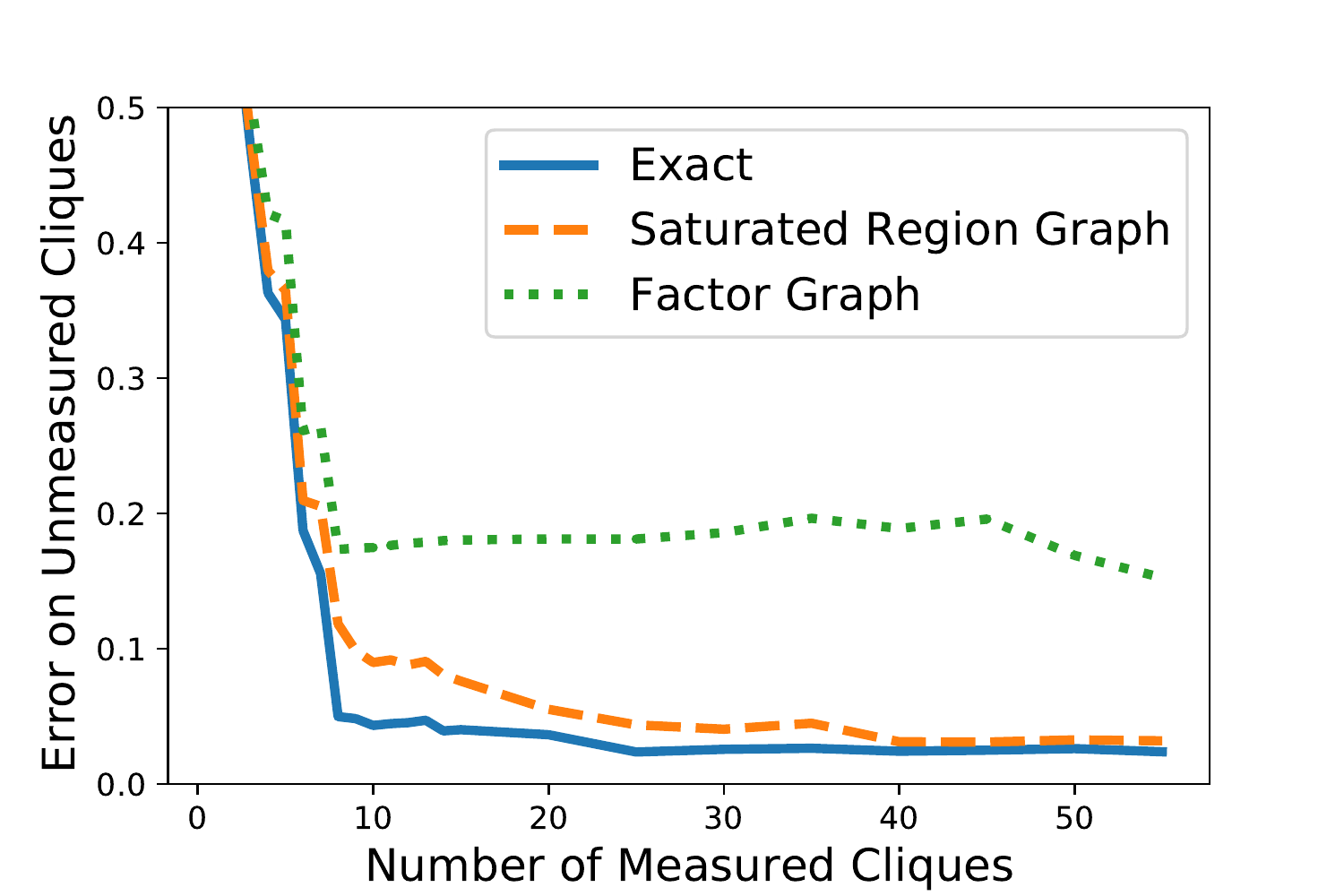}
    \caption{\label{fig:outclique} Out-of-Model Marginals}
  \end{subfigure}%
  \begin{subfigure}[b]{0.33\textwidth}
    \centering\includegraphics[width=\textwidth]{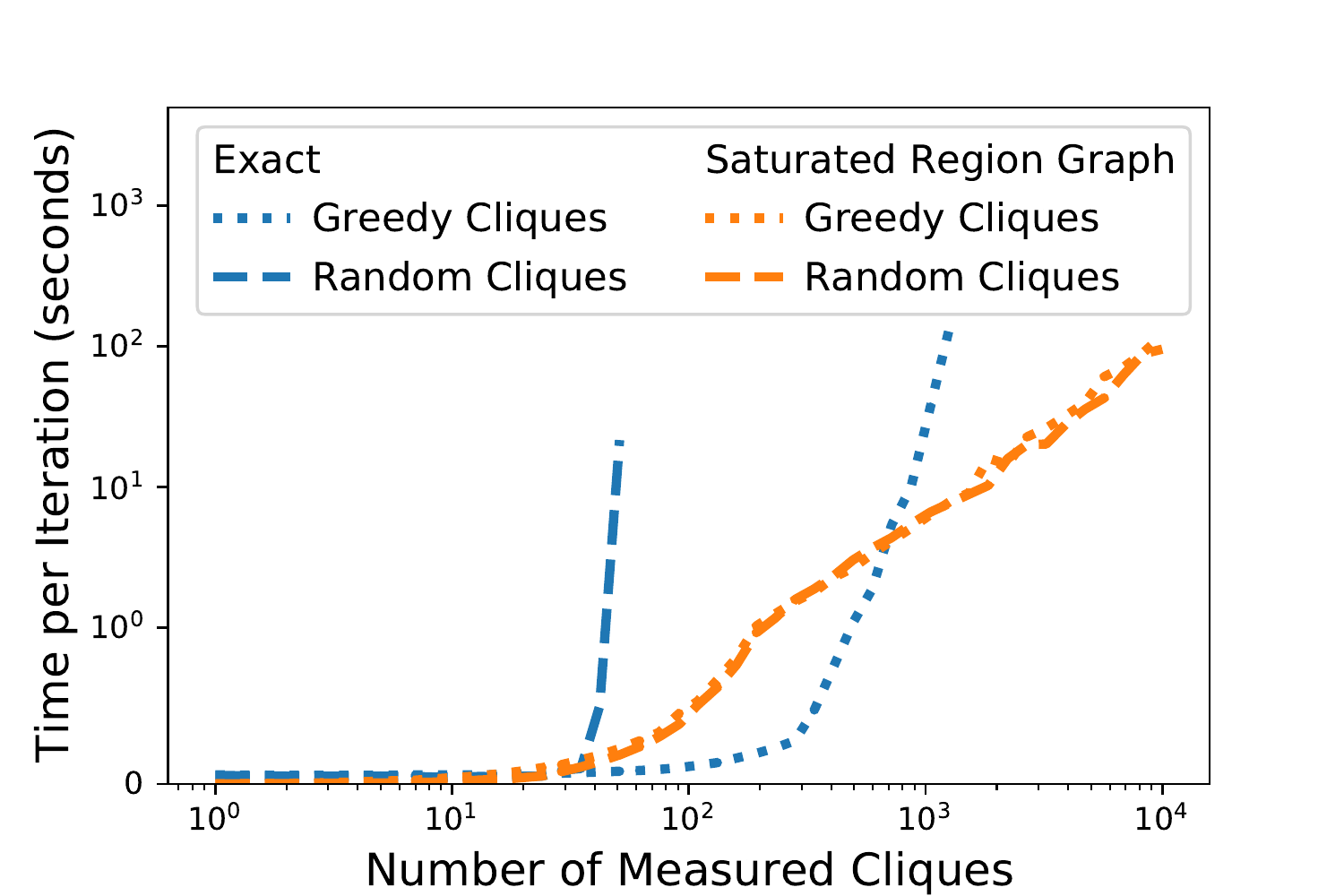}
    \caption{\label{fig:scalability} Scalability}
  \end{subfigure}%
  \caption{ \label{fig:exact_vs_approx} Comparison between \prox  with exact and approximate inference (for different region graph structures):  (a) error on in-model (measured) marginals, (b) error on out-of-model (unmeasured) marginals, and (c) scalability of \prox vs. number of measured marginals.
  }
\end{figure}

\textit{Scalability.} 
%How does \prox scale when using an exact \oracle (on a junction tree) and an approximate oracle (on a region graph).
Next we consider high-dimensional data and compare the scalability of \mysf{Exact} and \mysf{Saturated Region Graph} on $100$-dimensional data with $n_1 = \dots = n_{100} = 10$.  We vary $k$ from $1$ to $10^4$
%(spaced evenly in log-space),
and calculate the per-iteration time of \prox.\footnote{Scalability experiments were conducted on two cores of a machine with a 2.4GHz CPU and 16 GB of RAM.}
%$, which essentially measures the time of running \oracle and computing $\nabla L$.
We consider two schemes for selecting measured cliques: \emph{random} selects triples of attributes uniformly at random, and \emph{greedy} selects triples to minimize the junction tree size in each iteration.
As shown in \cref{fig:scalability}, \mysf{Exact} can handle about $50$ random measured cliques or $1000$ greedy ones before the per-iteration time becomes too expensive (the growth rate is exponential). In contrast, \mysf{Saturated Region Graph} runs with $10000$ measured cliques for either strategy and could run on larger cases (the growth rate is linear).  

\paragraph{Improving Scalability and Accuracy in Sophisticated Mechanisms}
We show next how \prox can be used to improve the performance of two sophisticated mechanisms for answering complex workloads of linear queries, including marginals. HDMM~\cite{mckenna2018optimizing} is a state-of-the-art algorithm that first selects a ``strategy'' set of (weighted) marginal queries to be measured, and then reconstructs answers to workload queries. MWEM~\cite{hardt2012simple} is another competitive mechanism that iteratively measures poorly approximated queries to improve a data distribution approximation. In each algorithm, the bottleneck in high dimensions is estimating a data distribution~$\hat\p$, which is prohibitive to do explicitly. \pgm can extend each algorithm to run in higher dimensions~\cite{mckenna2019graphical}, but still becomes infeasible with enough dimensions and measurements. HDMM is a ``batch'' algorithm and either can or cannot run for a particular workload. Because MWEM iteratively selects measurements, even in high dimensions it can run for some number of iterations before the graphical model structure becomes too complex. By using \apgm instead, HDMM can run for workloads that were previously impossible, and MWEM can run for any number of rounds. 
We use the \textsc{fire} dataset from the 2018 NIST synthetic data competition~\cite{ridgeway2021challenge}, which includes 15 attributes and $m \approx 300,\!000$ individuals.

For HDMM, we consider workloads of $k$ random 3-way marginals, for $k=1,2,4,8,\dots,256,455$, run five trials, and report root mean squared error, the objective function that HDMM optimizes.
%, as the domain is far too large\footnote{it is able to run for $k=2$ because the workload only involves $6$ attributes, rather than the full $15$.}.
%When $k$ is sufficiently small, HDMM can be run with \pgm. For larger $k$, HDMM can only be run by using \prox with an approximate marginal oracle.  
\cref{fig:hdmm} shows the results.
HDMM with a full data vector cannot run for $k>2$, but we can still analytically compute the \emph{expected error} if it were able to run.
HDMM\,+\,\mysf{Exact} fails to run beyond $k=16$, while HDMM\,+\,\mysf{Region Graph} is able to run in every setting, substantially expanding the range of settings in which HDMM can be used.
Both variants offer the significant error improvements of PGM-style inference, because they impose non-negativity constraints that reduce error. For example,  when $k=455$, there is a $3 \times $ reduction in RMSE.

%This is expected: HDMM does not impose non-negativity constraints while we do.  Since the true data is non-negative, it is not surprising that imposing non-negativity leads to better performance.  The magnitude of the improvement grows with the number of measured cliques.  

%  In summary, our proposed approach has allowed HDMM to run in settings where it was not previously possible, \emph{and} it gives lower error than HDMM expects even if it could be run, a win-win situation.  

%\item {\sc EXACT}(Junction tree)	
%\item {\sc APPROX}(OSat region graph)
%\item {\sc APPROX}(Sat region graph)	
%\item {\sc APPROX}(Factor graph)	

\begin{figure}
  \centering
  \begin{subfigure}[b]{0.33\textwidth}
    \centering\includegraphics[width=\textwidth]{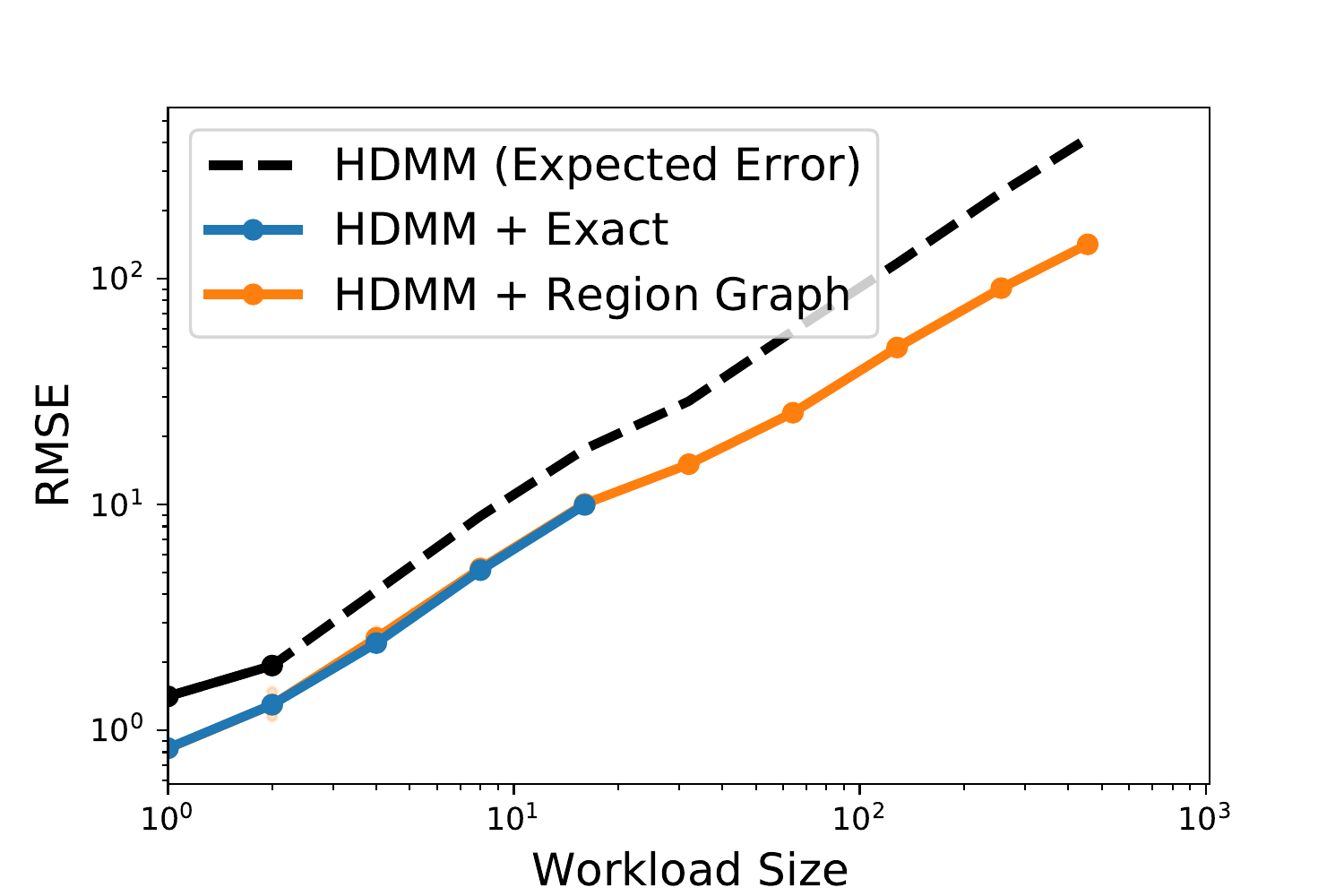}
    \caption{\label{fig:hdmm} HDMM -- \textsc{fire}}
  \end{subfigure}%
  \begin{subfigure}[b]{0.33\textwidth}
    \centering\includegraphics[width=\textwidth]{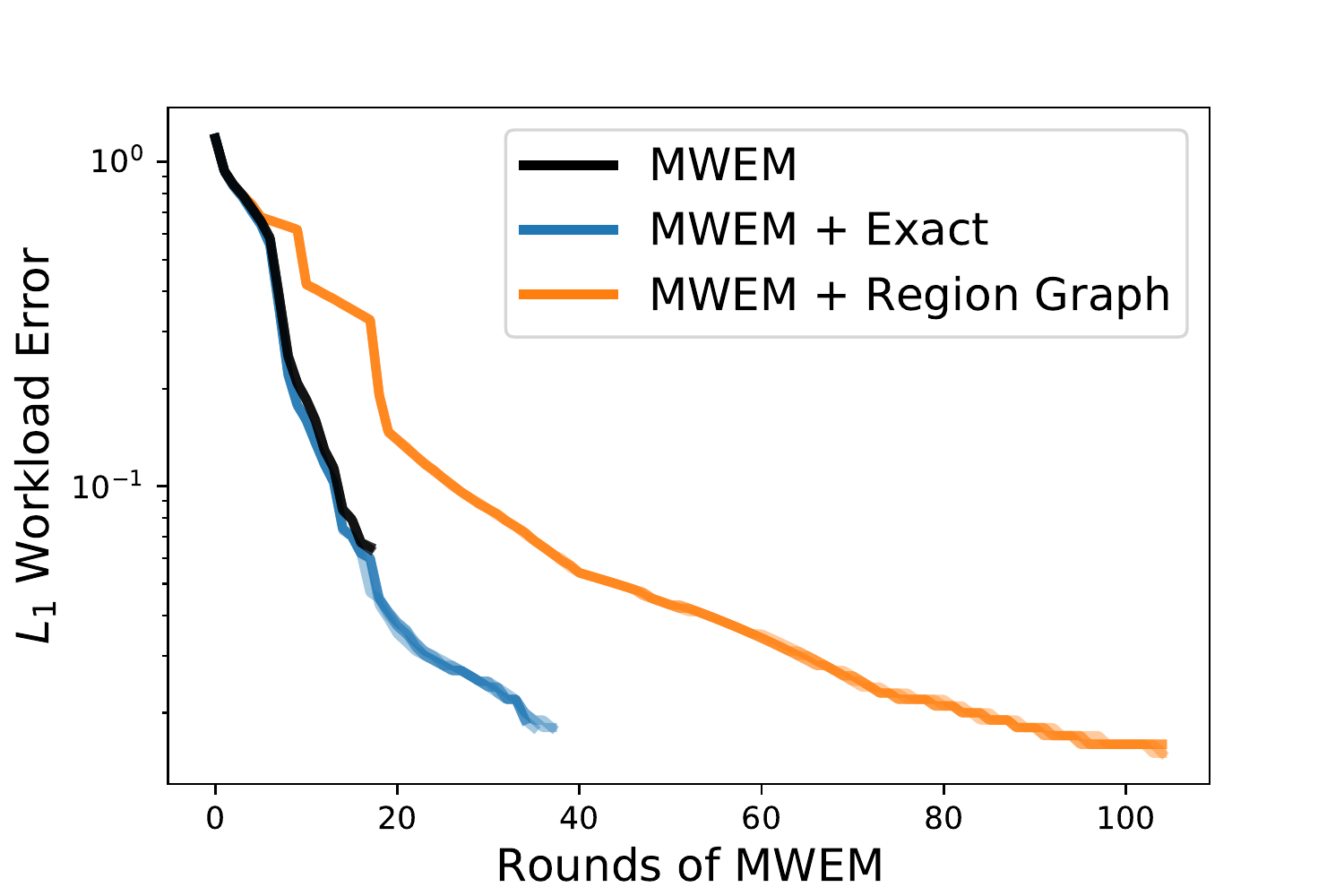}
    \caption{\label{fig:mwem} MWEM -- \textsc{fire}}
  \end{subfigure}%
  \begin{subfigure}[b]{0.33\textwidth}
    \centering\includegraphics[width=\textwidth]{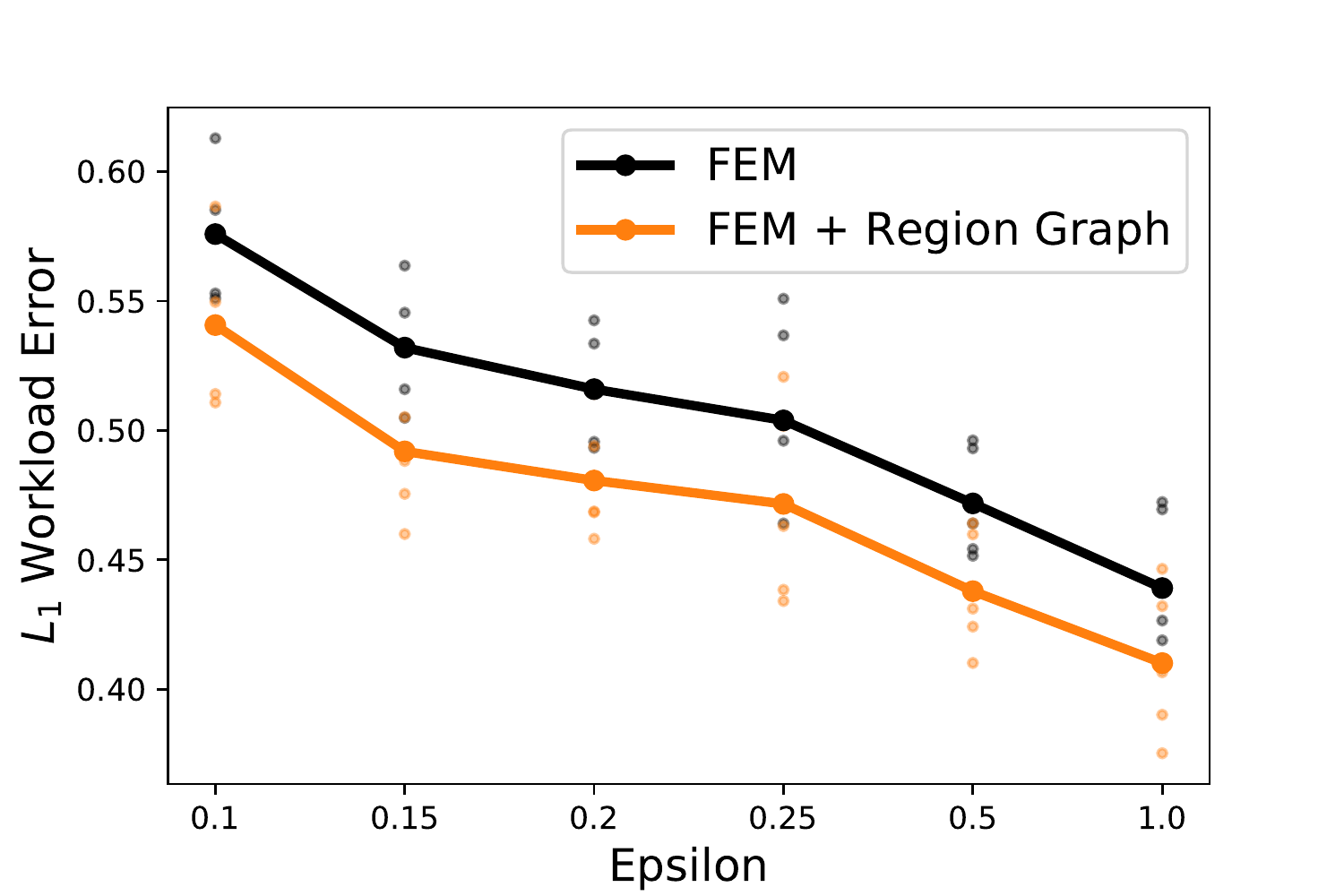}
    \caption{\label{fig:fem} FEM -- \textsc{adult}}
  \end{subfigure}%
\caption{Three examples of using \prox to improve scalability and accuracy of other algorithms.}
\end{figure}

%\subsection{Improving Scalability of MWEM}
%We now demonstrate how we can integrate with and improve the scalability of MWEM.  For a workload of marginals, MWEM maintains an estimated distribution (initially uniform) and runs for a number of rounds.  In each round, MWEM selects the worst-approximated clique (using the exponential mechanism), and measures the corresponding marginal (using the Laplace mechanism), and updates the estimated data distribution (using the multiplicative weights formula).  For high-dimensional data and complex workloads with many overlapping cliques, MWEM fails to scale after a few rounds.  MWEM can be run for more rounds by replacing the multiplicative weights update with \pgm, but even this eventually fails to scale as well, as the size of the junction tree grows exponentially with the number of rounds/cliques.  In this section, we show that we can run MWEM for longer by using \pgm with an approximate marginal oracle instead.

For MWEM, we consider the workload of all $2$-way marginals, use a privacy budget of $\epsilon = 0.1$ \emph{per round}, and run for as may rounds as possible, until MWEM has measured all $2$-way marginals or exceeds a generous time/memory limit of 24 hours and 16GB. \cref{fig:mwem} shows the results. As expected, MWEM\,+\,\mysf{Exact} runs successfully in early iterations, but exceeds resource limits by 35--40 rounds. In comparison, MWEM\,+\,\mysf{Region Graph} can run to completion and eventually measure all 2-way marginals. For a fixed number of rounds, \mysf{Exact} has lower error, in this case, substantially so, but results are data-dependent (we evaluate with other datasets in \cref{sec:extraexp}). The difference can largely be traced to \mysf{Exact}'s better performance on out-of-model cliques. In contrast, HDMM's measurements support all workload queries, so no out-of-model inference is required, and we see little gap between exact and approximate inference. Improving performance of approximate inference for out-of-model marginals is an area to be considered for future work; see \cref{sec:extraout}.
 
\textbf{Additional experiments\ } We also apply \appgm to improve the accuracy of FEM, a recent state-of-the-art query-answering mechanism~\cite{vietri2020new}. The setup is similar to the DualQuery experiment in~\cite{mckenna2019graphical}: we run FEM to completion to release $\y$, but instead of answering queries directly with $\y$, we use \appgm to estimate pseudo-marginals, from which we compute query answers. This leads to a modest error reduction for all $\epsilon$ (\cref{fig:fem}; details in \cref{sec:extraexp}). In \cref{sec:extraexp} we also compare \pgm and \appgm directly to a method proposed in PriView~\cite{qardaji2014priview} for estimating consistent marginals, and find that PGM-based methods are more accurate for almost all values of $\epsilon$.  We additionally compare \pgm and \appgm to the recent Relaxed Projection method~\cite{aydore2021differentially}, and found that \appgm performs better for $\epsilon > 0.1$, although is outperformed for $\epsilon \leq 0.1$.  

\section{Practical Considerations and Limitations} \label{sec:limitations}
When using our approach in practice, there are several implementation issues to consider.  First note that \gbp is an iterative algorithm that potentially requires many iterations to solve \cref{prob:energy}, and this marginal inference routine is called within each iteration of \prox.  Our theory requires running \gbp until convergence, but in practice we only need to run it for a fixed number of iterations.   In fact, we find that by \emph{warm starting} the messages in \gbp to the values from the previous call, we can actually run only one inner iteration of \gbp within each outer iteration of \prox.  This approach works remarkably well and makes much faster progress on reducing the objective than using more inner iterations.  The main drawback of this is that we can no longer rely on a line search to find an appropriate step size within \prox, since one iteration of \gbp with warm starting is not necessarily a descent direction.  Using a constant step size works well most of the time, but selecting that step size can be tricky.  We utilize a simple heuristic that seems to work well in most settings, but may require manual tuning in some cases.  Second, we observed that introducing damping into \gbp improved its stability and robustness, especially for very dense region graphs.  Finally our MWEM experiment revealed that it is better to use \pgm over \appgm in the context of an MWEM-style algorithm, even though \pgm is more limited in the number of rounds it can run for. This performance difference can be traced back to our method for out-of-model inference, where utilization of local information only can lead to poor estimates.  We discuss alternative approaches for this problem in \cref{sec:extraout}.

\section{Related Work} \label{sec:related}

A number of recent approaches support reconstruction for measurements on high-dimensional data. As part of the PriView algorithm~\cite{qardaji2014priview}, the authors describe a method for resolving inconsistencies in noisy measured marginals, which has since been incorporated into other mechanisms~\cite{chen2015differentially,zhang2018calm,chowdhury2020data,zhang2021privsyn}. Like \apgm, their method only guarantees local consistency. In \cref{sec:extraexp}, we show empirically that it achieves similar (but slightly worse) error than \apgm. In addition, the method is less general, as measured queries may only be marginals, while \apgm allows an arbitrary convex loss function to be specified over the marginals. This extra generality is critical for  integrating with mechanisms like HDMM, FEM, and MWEM when the workload contains more complex linear queries.

%One notable method was introduced in PriView \cite{qardaji2014priview}.  Priview derived a special purpose method to resolve inconsistencies in noisy marginals.  Their method has been incorporated into a variety of mechanisms that deal with marginals \cite{chen2015differentially,zhang2018calm,chowdhury2020data,zhang2021privsyn}.   Like the method we propose here, their method only guarantees local consistency, rather than global consistency as \pgm does.  While it is inspired by a particular optimization problem (similar to \cref{prob:convex}), to our knowledge it is not guaranteed to converge to the optimal solution of that problem.  Nevertheless, the method works well in practice.  We empirically compare against this method in the appendix, showing that it achieves similar (but slightly worse) error than our proposed approach.  In addition, their  approach is not quite as general, as we allow an arbitrary convex loss function to be specified over the marginals, whereas they expect to observe the noisy marginals directly.  This extra generality matters when integrating with mechanisms like HDMM, FEM, and MWEM for workloads containing more complicated queries than marginals.

In concurrent work, Aydore et al.~\cite{aydore2021differentially} and Liu et al.~\cite{liu2021iterative} proposed scalable instantiations of the MWEM algorithm, both avoiding the data vector representation in favor of novel compact representations.  Although originally described in the context of MWEM-style algorithms, the key ideas presented in these works can be abstracted to the more general setting considered in this work.  Specifically, these methods can be seen as alternatives to \appgm for overcoming the scalability limitations of \pgm; each of these methods make different approximations to overcome the inherent hardness of \cref{prob:marginal}.  A direct comparison between \pgm or \appgm, and these alternatives remains an interesting question for future research.  

To avoid the data vector representation, Liu et al.~\cite{liu2021leveraging} restrict the support of the data vector to the domain elements that appear in a public dataset. This is much more scalable, but could result in poor performance if the public domain and the input domain differ substantially. %This reduces the size of the data vector to (at most) the number of records in the public dataset, and is much more manageable.  However, it is a risky assumption if the domain is different in the sensitive data. 

Lastly, Dwork et al.~\cite{dwork2015efficient} propose an algorithm similar to \appgm.
%for privately releasing marginals using convex relaxations.
Their approach also projects onto an outer approximation of the marginal polytope, using the Frank Wolfe algorithm. The outer approximation is constructed via geometric techniques and is different from the local polytope we consider. They prove favorable error bounds with polynomial running time, but leave open the implementation and evaluation of their approach. By using the local polytope and message-passing algorithms, our method can scale in practice to problems with millions of variables.

%show theoretically that their algorithm solves the optimization problem in polynomial time, but leave implementation and evaluation as future work.
%do not give a practical implementation of it that can exploit structure in the constraint set or scale to problems with millions of variables.  

\medskip

\section*{Acknowledgements}

This work was supported by the National Science Foundation under grant IIS-1749854, by DARPA and SPAWAR under contract N66001-15-C-4067, and by Oracle Labs, part of Oracle America, through a gift to the University of Massachusetts Amherst in support of academic research.

%This work was supported by Oracle, the National Science Foundation under grant IIS-1749854 and by DARPA and SPAWAR under contract N66001-15-C-4067.

\bibliography{refs}
\bibliographystyle{unsrtnat}
%\newpage
%\input{checklist}

\newpage
\appendix
\section{Approximate Marginal Inference Algorithm} \label{sec:approxalg}

\begin{algorithm}
%\caption{Message Passing for Convex Free Energy on a Region Graph \cite{hazan2012tightening}} \label{alg:mpconvex}
\caption{\gbp: Convex Generalized Belief Propagation \cite{hazan2012tightening}} \label{alg:mpconvex}
\begin{algorithmic}
\STATE {\bfseries Input:} Region Graph $G=(\V,\E)$, parameters $\btheta = (\btheta_r)_{r \in \V}$, convex counting numbers $\counting_r > 0$
\STATE {\bfseries Output:} Model marginals $\btau = (\btau_r)_{r \in \C}$
\STATE $\counting_{r,t} = \counting_r / (\counting_t + \sum_{p \rightarrow t} \counting_p)$
\STATE Initialize $m_{r\rightarrow t}(\x_t) = 0$ and $\lambda_{t \rightarrow r}(\x_t) = 0$
\FOR{$i = 1,\dots$}
\FOR{$r \rightarrow t$}
\STATE $m_{r \rightarrow t}(\x_t) = \counting_r \log{\Big( \sum_{\x_r \setminus \x_t} \exp{\Big( (\btheta_r(\x_r) + \sum_{c \neq r} \lambda_{c \rightarrow r}(\x_c) - \sum_{p} \lambda_{r \rightarrow p}}(\x_p))/\counting_r)\Big)\Big)}$
\STATE $\lambda_{t \rightarrow r}(\x_t) = \counting_{r,t} \Big( \btheta_t(\x_t) + \sum_{c} \lambda_{c \rightarrow t}(\x_t) + \sum_{p} m_{p \rightarrow t}(\x_t) \Big) - m_{t \rightarrow r}(\x_t)$
\ENDFOR
\ENDFOR
\FOR{$r \in \C$}
\STATE $\btau_r(\x_r) \propto \exp{\Big((\btheta_r(\x_r) + \sum_{t} \lambda_{t \rightarrow r}(\x_t) - \sum_p \lambda_{r \rightarrow p}(\x_r))/\counting_r\Big)}$
\ENDFOR
\RETURN $\btau = (\btau_r)_{r \in \V}$
\end{algorithmic}
\end{algorithm}

%\section{Convex Duality}
%\newcommand{\z}{\mathbf{z}}
%\newcommand{\bpi}{\boldsymbol{\pi}}
%\newcommand{\bsigma}{\boldsymbol{\sigma}}
%\newcommand{\grad}{\nabla}
\renewcommand{\v}{\mathbf{v}}
\renewcommand{\u}{\mathbf{u}}

\begin{lemma} \label{lem:cvxdual}
Let $G$ be a region graph, let $\counting$ be positive counting numbers, and suppose $\hat\btau = \gbp(G, \hat\btheta, \kappa)$ for parameters $\hat\btheta$.  Then, for any vector $\z$:
$$
\argmin_{\btau \in \L(G)} -\btau^\top \left(-\grad H_\counting(\hat\btau) + \z\right) - H_\counting(\btau) 
= \argmin_{\btau \in \L(G)} -\btau^\top \big(\hat\btheta + \z\big) - H_\counting(\btau)
$$
\end{lemma}

For the remainder of this section, let $S$ be the linear subspace parallel to the affine hull of $\L(G)$, and let $S_\perp$ be the orthogonal complement of $S$. That is, if we write $\L(G) = \{\btau \geq \mathbf{0}: A \btau = \mathbf{b}\}$ using the constraint matrix $A$, then $S$ is the null space of $A$. This means that for any $\mathbf{0} < \btau \in \L(G)$ and $\z \in S$, there is some $\lambda > 0$ such that $\btau + \lambda \z \in \L(G)$. On the other hand, if $\z \in \L(G)$ and $\z \notin S$, there is no $\lambda > 0$ such that $\btau + \lambda \z \in \L(G)$.

\begin{proof}
  Because $\hat\btau = \gbp(G, \hat\btheta, \counting)$ we know that $\hat\btau$ minimizes $-\tilde\btau^\top\hat\btheta - H_\counting(\tilde\btau)$ over all $\tilde\btau \in \L(G)$, and it is easy to see from the final line of \gbp that $\hat\btau > \mathbf{0}$. Therefore, we can apply \cref{lem:inverse-mapping} below to conclude that there are vectors $\v_\perp, \v_\perp' \in S_\perp$ such that
$$
\begin{aligned}
-\grad H_\counting(\hat\btau) &= \u(\hat\btau) + \v_\perp \\
\hat\btheta &= \u(\hat\btau) + \v_\perp' 
\end{aligned}
$$
where $\u(\hat\btau)$ is the projection of $-\grad H_\counting(\hat\btau)$ onto $S$.

We will now show that the linear parts of the objectives of the two minimization problems in the lemma statement differ by only a constant. Since the nonlinear part is the same, this will prove that the objectives as a whole differ by only a constant, so the problems have the same minimizers, as stated.

Let $\z = \z_\parallel + \z_\perp$ where $\z_\parallel \in S$ and $\z_\perp \in S_\perp$. Then, for any $\btau \in \L(G)$, the linear component of the first objective is
\begin{equation}
  \label{eq:dot1}
  -\btau^\top(-\nabla H_\kappa(\btau) + \z) = -\btau^\top(\u(\hat\btau)+ \v_\perp + \z_\parallel + \z_\perp) = -\btau^\top(\u(\hat\btau) + \z_\parallel) + c
\end{equation}
where $c = -\btau^\top(\v_\perp  + \z_\perp)$ is a constant that does not depend on $\btau$, since, for any $\btau, \btau' \in \L(G)$ we have
$$
\btau'^\top (\v_\perp + \z_\perp) - \btau^\top (\v_\perp + \z_\perp)  = (\btau' -\btau)^\top (\v_\perp + \z_\perp) = 0,
$$
since $\btau' - \btau \in S$ and $\v_\perp + \z_\perp \in S_\perp$. 

Similarly, the linear component of the second objective is
\begin{equation}
  \label{eq:dot2}
  -\btau^\top(\hat\btheta + \z) = -\btau^\top(\u(\hat\btau) + \v'_\perp + \z_\parallel + \z_\perp) = -\btau^\top(\u(\hat\btau) + \z_\parallel) + c'
\end{equation}
where $c' = -\btau^T(\v'_\perp + \z_\perp)$ is a (different) constant independent of $\btau$.

This shows that the objectives differ by a constant, and completes the proof. 

\cref{eq:dot1} and \cref{eq:dot2} show that the objectives differ by a constant, which completes the proof. 
\end{proof}

\begin{lemma} \label{lem:inverse-mapping} Let $G$ be a region graph, let $\counting$ be positive counting numbers, and let $\btau \in \L(G)$ with $\btau > \mathbf{0}$. Define $\Theta(\btau) = \{\btheta: \btau = \argmin_{\tilde\btau \in \L(G)} -\tilde\btau^\top \btheta - H_\counting(\tilde\btau)\}$ to be the set of all $\btheta$ such that $\gbp(G, \btheta, \counting) = \btau$. Then 
$$
\Theta(\btau) = \u(\btau) + S_\perp
$$
where $\u(\btau)$ is the projection of  $-\nabla H_\counting(\btau)$ onto $S$. 
\end{lemma}

\begin{proof}
  This follows fairly standard arguments in convex analysis after noting that the objective of the optimization problem coincides with the convex conjugate of $-H_\counting$ (e.g., see Rockafellar, 2015\footnote[1]{Rockafellar, R. T. (2015). \emph{Convex analysis}. Princeton university press.}; Bertsekas, 2009\footnote[2]{Bertsekas, Dimitri P. Convex optimization theory. Belmont: Athena Scientific, 2009.}), but with some specialization to our setting.

Define $f(\btau)$ to be the extended real-valued function that takes value $- H_\counting(\btau)$ for $\btau \in \L(G)$ and $+\infty$ for $\btau \notin \L(G)$. Let $\partial f(\btau)$ be the subdifferential of $f$ at $\btau \in \L(G)$. 

We will first show that $\Theta(\btau) = \partial f(\btau)$. 

By the definition of a subgradient, for $\btau \in \L(G)$,
$$
\begin{aligned}
\btheta \in \partial f(\btau) 
&\iff f(\tilde\btau) \geq f(\btau) + \btheta^\top (\tilde\btau - \btau)\quad  \forall \tilde\btau \in \R^d \\
&\iff f(\tilde\btau) \geq f(\btau) + \btheta^\top (\tilde\btau - \btau)\quad  \forall \tilde\btau \in \L(G) \\
&\iff \btau^\top \btheta - f(\btau)  \geq  \tilde\btau^\top \btheta - f(\tilde\btau) \quad \forall \tilde\btau \in \L(G) \\
&\iff \btau = \argmax_{\tilde\btau \in \L(G)} \tilde\btau^\top \btheta - f(\tilde\btau) \\
&\iff \btau = \argmin_{\tilde\btau \in \L(G)} - \tilde\btau^\top \btheta - H_\kappa(\tilde\btau) \\
&\iff \btheta \in \Theta(\btau).
\end{aligned}
$$
In the second line, we used the fact that the inequality always holds for $\tilde\btau \notin \L(G)$ because $f(\tilde \btau) = +\infty$ and the other quantities are finite. In the third line, we used the fact that $f(\tilde \btau)$, which coincides with $-H_\counting(\tilde \btau)$ on $\L(G)$, is strictly convex, so $\btau$ is a \emph{unique} maximizer of $\tilde\btau^\top \btheta - f(\tilde\btau)$.

Now, we will show that $\partial f(\btau) = \u(\btau) + S_\perp = \{ \u(\btau)+ \v : \v \in S_\perp\}$, which will conclude the proof.

We use the following  characterization of the subdifferential (Rockafellar, 2015, Theorem 23.2):
\begin{equation}
  \label{eq:rockafellar}
\btheta \in \partial f(\btau) 
\iff \btheta^\top \z \leq f'(\btau; \z) \quad \forall \z \in \R^d
\end{equation}
where $f'(\btau; \z)$ is the directional derivative of $f$ along direction $\z$. Since $f$ is the restriction of the differentiable function $-H_\counting$ to $\L(G)$, its directional derivatives coincide with those of $-H_\counting$ for points $\btau \in \L(G)$ with $\btau > \mathbf{0}$ and directions in $\z \in S$ (so that $\btau + \lambda \z \in \L(G)$ for small enough $\lambda > 0$), and are equal to $+\infty$ for points $\btau \in \L(G)$ and directions $\z \notin S$ (so that $\btau + \lambda \z \notin \L(G)$ for any $\lambda > 0$). That is, for $\btau \in \L(G)$ with $\btau > 0$, 
\begin{equation}
  \label{eq:directional}
f'(\btau; \z) = 
\begin{cases}
-\grad H(\btau)^\top \z & \z \in S\\
+\infty & \z \notin S
\end{cases}
\end{equation}
Therefore, by \cref{eq:rockafellar} and \cref{eq:directional}, for any $\btau \in \L(G)$ with $\btau > \mathbf{0}$,
$$
\begin{aligned}
\btheta \in \partial f(\btau)
&\iff \btheta^\top \z \leq f'(\btau; \z) &&\forall \z \in \R^d \\
&\iff \btheta^\top \z \leq  f'(\btau; \z) &&\forall \z \in S \\
&\iff \btheta^\top \z \leq -\grad H(\btau)^\top\z &&\forall\z \in S \\
&\iff \btheta^\top\z = -\nabla H(\btau)^\top \z &&\forall \z \in S \\
&\iff \btheta = \u(\btau) + \v &&\v \in S_\perp,
\end{aligned}
$$
where $\u(\btau)$ is the projection of $-\grad H(\btau)$ onto $S$. 
In the second line, we used the fact that the inequality always holds for $\z \notin S$ because $f'(\btau; \z) = + \infty$ and the other quantities are finite.
In the fourth line, we observed that $\z \in S$ iff $-\z \in S$ (since $S$ is a linear subspace) and 
$$
 \btheta^\top(-\z) \leq -\grad H(\btau)^\top(-\z) \iff \btheta^\top \z \geq -\grad H(\btau)^\top \z,
$$
so the third line is equivalent to both inequalities holding for all $\z \in S$. The equivalence of the final line to the penultimate line is a straightforward exercise by breaking both $\btheta$ and $-\grad H(\btau)$ into their orthogonal components along $S$ and $S_\perp$, respectively, and observing that the component of $-\grad H(\btau)$ along $S$ is $\u(\btau)$.
\end{proof}

\section{Out-of-Model Inference} \label{sec:extraout}

We now turn our attention to the problem of out-of-model inference; i.e., estimating $\btau_r$ where $r \not \in \V$.  There are many approaches for this problem that seem natural on the surface, but upon close inspection each one has it's problems. In \cref{sec:approach}, we proposed one approach that had certain desirable properties, but we considered many alternatives which enumerate below and discuss in detail.  

\subsection{Variable Elimination in $\p_{\btheta}$}

In \pgm, out-of-model inference was done by performing variable elimination in the graphical model $\p_{\btheta}$.  There are two problems with applying that idea here.  First, variable elimination will not in general be tractable for the graphical models we may encounter, since it is an exact inference method.  Second, if we run variable elimination to estimate in-model marginals from $\btheta$ produced by \prox, it will give a different answer than the pseudo-marginals $\btau$ produced by \prox (even if $\btau \in \M(\V)$ is a realizable marginal).  In this case, the pseudo-marginals estimated by \prox are the ones that should be trusted, and the parameters $\btheta$ are only useful in the context of our approximate marginal oracle \gbp.   In summary, this approach is not viable, and even if it was, it has undesirable properties.

Before moving on, we make note of an alternate way to perform exact out-of-model inference that will motivate our first approach to approximate out-of-model inference. They key idea is to add a new zero log-potentials $\btheta_r = \mathbf{0}$ for the new clique whose marginal we are interested in estimating.  Clearly, the introduction of this zero log-potential does not change the distribution $\p_{\btheta}$ or it's in-model marginals $\bmu_{\btheta}$.  However, when we run an exact \oracle with these new parameters, it will produce all in-model marginals, and the new out-of-model marginal as well. 

%even though \prox finds the optimal $\bmu$ no matter what convex counting numbers are used in the entropy approximation, different choices of counting numbers will lead to different parameter vectors $\btheta$.   The true marginals of the graphical model with parameter $\btheta$ will be different from the ones estimated in \prox \footnote{in this case, the true marginals are not the correct ones}, and the difference may be substantial for some choices of convex counting numbers $\counting$.  Thus, this idea is neither viable nor desirable.  

\subsection{Running \gbp on an Expanded Region Graph}

Using the idea above, one approach to out-of-model inference is to expand the region graph to include the region $r$ whose pseudo-marginal we are interested in.  This will require adding at least one new vertex $r$ to the region graph.  Edges and additional vertices could be added depending on the structure of the existing region graph, and the desired local consistency constraints that $\btau_r$ should obey.  With this new region graph, we can set $\btheta_r = \mathbf{0}$ (and do the same for any additional vertices we added as well), and run \gbp on the new graph.  This is an interesting idea, but it leaves open several questions:

\begin{enumerate}
\item What nodes and edges should be included in the augmented region graph?
\item What counting numbers should be assigned to those nodes? 
\item What formal guarantees can we make about this approach?
\item Can we analyze the message-passing equations in \gbp to perform an equivalent computation without re-running the algorithm in its entirety?
\end{enumerate}

For question (1) above, a natural choice is to use the same structure as the original region graph.  For example, if the original region graph is a factor graph, then we can simply add one new vertex corresponding to the new one, and add edges connecting to the singleton cliques.  If the original region graph is saturated, then we can build a new saturated region graph that includes the new clique.  

For question (2) above, a natural choice is to use $\counting'_r=1$ for all regions $r$ (including the new region), since that is the scheme used to set $\counting$ within $\prox$.  Unfortunately, the new pseudo-marginals $\btau' = \gbp(\btheta', \counting')$ may not agree with the originally optimized pseudo-marginals $\btau = \gbp(\btheta, \counting)$ on the in-model cliques.  Specifically, $\btau_r$ need not equal $\btau'_r$ when $r \in \V$.  This is clearly undesirable, and would be a consistency violation.  A better choice of the counting numbers would be $ \kappa'_r = \kappa_r$ for $ r \in \V$ and  $\kappa_{r'} = 0$ otherwise.  As we show below, this approach has a compelling theoretical guarantee.

\begin{theorem} \label{thm:project}
Let $G=(\V,\E)$ be a region graph, $\btheta$ be parameters, $\counting$ be positive counting numbers and let $\btau = \gbp(G, \btheta, \counting)$.  Now let $G'=(\V',\E')$ be a region graph that extends $G$ (i.e, $\V \subseteq \V'$ and $\E \subseteq \E'$), $\btheta'_r = \btheta_r$ if $r \in \V$ and $\btheta'_r = \mathbf{0}$ if $r \not \in \V$, $\counting'_r = \counting_r$ for $ r \in \V$ and $\counting'_r = \varepsilon$ otherwise.

$$ \btau' = \lim_{\substack{\varepsilon \rightarrow 0^+}} \gbp(G', \btheta', \kappa') $$

If $S = \set{\btau' \in L(G') \mid \btau'_r = \btau_r \forall r \in \V} \neq \varnothing $, then

$$\btau' = \argmax_{\btau' \in S} \sum_{r \in \V' \setminus \V} H(\btau'_r) $$
\end{theorem}

\begin{proof}
We begin by restating the free energy minimization problem solved by \gbp.
\begin{align*}
\bmu' &= \argmin_{\btau' \in \L(G')} -\btheta^{\top} \btau' -  H_{\kappa'}(\btau') \\\
&= \argmin_{\btau' \in \L(G')} - \Big[\sum_{r \in \V} \btheta_r^{\top} \btau' + \kappa_r H(\btau'_r) \Big] - \Big[\sum_{r \in \V' \setminus \V} \mathbf{0}^{\top} \btau' + \varepsilon H(\btau'_r) \Big] \\
&= \argmin_{\btau' \in \L(G')} - \Big[\sum_{r \in \V} \btheta_r^{\top} \btau' + \kappa_r H(\btau'_r) \Big] - \varepsilon \sum_{r \in \V' \setminus \V} H(\btau'_r) \\
\end{align*}

Note that as $\varepsilon \rightarrow 0$, the objective only depends on $\btau_r$ for $r \in \V$ (and not $r \in \V' \setminus \V$).  Thus, $\btau_r$ only affect the problem via the constraints they impose on the problem.  Since $\btau$ is the optimizer of the relaxed problem when $L(G')$ is replaced by $L(G)$ (which includes a subset of the constraints), if $\btau$ is feasible in the larger problem (which it is by assumption $S \neq \varnothing$), it is also optimal in this problem.  Moreover, since we are taking the limit as $\varepsilon \rightarrow 0$ from the right, there will be an infinitesimally small entropy penalty, which will force $\bmu'_r$ to have maximum entropy among marginals that are consistent with $\bmu$, as desired.
\end{proof}

\cref{thm:project} is a compelling reason to use this approach, namely running \gbp with zero counting numbers for the new cliques whose marginals we are estimating.   One subtle detail to this theorem is that it is certainly possible that $ S = \varnothing$, which means that there aren't any pseudo-marginals in the expanded region graph that are consistent with the pseudo-marginals in the original region graph.  In this case, it is not immediately clear how to characterize the behavior of this approach.  

\begin{remark}[Special Case: Factor Graph]
In the special case when both the original and expanded region graphs are factor graphs, we can guarantee that $S \neq \varnothing$ and we can efficiently estimate the new pseudo-marginal without rerunning \gbp over the full graph.  Since factor graphs only require each pseudo-marginal to be internally consistent with respect to the one-way marginals, we can always find higher-order marginals by multiplying the one-way marginals.  Clearly, this gives the maximum entropy estimate for the new pseudo-marginal that is internally consistent with the existing ones.  This is a computationally cheap estimate: it simply requires multiplying one-way marginals and does not require any iterative message passing scheme.  
\end{remark}

For more complex region graphs, things do not work out so nicely.  Since we are mainly interested in saturated region graphs in this work, this nice result for factor graphs is not particularly useful for our purposes.  In practice, there is a problem with running \gbp with a zero or near-zero counting numbers.  We observed empirically that using small counting numbers severely deteriorates the convergence rate of \gbp, and for that reason, this is not an ideal approach.  

\subsection{Minimizing Constraint Violation and Maximizing Entropy}

While the method described above has some drawbacks in practice, the principles underlying the approach are still sound: namely, we should find the maximum entropy distribution for the new marginal that is consistent with the existing marginals (for some natural notion of consistency).  However, for complex region graphs, it is certainly possible that no such marginals exist.  In that case, a natural alternative would be to find a pseudo-marginal that minimizes the constraint violation, and among all minimizers, has maximum entropy.  This is the approach that we evaluated empirically, and described in \cref{sec:approach}.  

It requires solving a quadratic minimization problem over the probability simplex.  This problem can be readily solved with iterative proximal algorithms like entropic mirror descent \cite{beck2003mirror}.   Entropic mirror descent guarantees the solution found will have maximum entropy among all minimizers of the objective.  Thus, when $S \neq \varnothing$, this method gives the same answer as \cref{thm:project}.  However, it is more general, and also does something principled when $S = \varnothing$.  Additionally, this method does not require any information about attributes not in $r$, and even though it is an iterative algorithm, each iteration runs much faster than an iteration of $\gbp$.  

\subsection{Running \prox over expanded local polytope.}  While the idea above is more principled than the alternatives that preceded it, it is still not ideal because it does not guarantee perfect consistency between the in-model pseudo-marginals and the out-of-model pseudo marginals.  When perfect consistency is not achievable, it settles for minimizing the constraint violation. 

We can overcome this limitation by running \prox on an \emph{over-saturated} region graph.  That is the region graph will contain vertices for every region necessary to define the loss function, \emph{and} every region whose pseudo-marginal we are interested in estimated.  The additional regions do not affect the loss function (the log-potentials will always remain $\mathbf{0}$), but it does impact the constraints.  In particular, upon convergence, the estimated pseudo-marginals will all be locally consistent.  This comes at a cost, however.  Since the region graph contains more vertices and edges, each iteration of \prox requires more time, and the algorithm as a whole is slower.  Whether it makes sense to use this strategy depends on how important perfect consistency is, as well as how many new marginals must be answered.  In our empirical evaluation of this approach, we found that it did produce better estimates than the previous idea, but also took considerably longer.  

\subsection{Incorporating Global Information}

As we saw empirically in \cref{sec:experiments}, our approach to out-of-model inference did not perform particularly well compared to the exact method used in \pgm.  We conducted more experiments to verify this in \cref{sec:extramwem}.  In this subsection, we explore in greater detail why it did not perform well in all cases.

Consider a simple graphical model with cliques $\C = \set{\set{A,B}, \set{B,C}}$, and suppose that $A$ is highly correlated with $B$ and $B$ is highly correlated with $C$.  Then clearly, $A$ and $C$ should also be highly correlated.  When performing exact inference in this model, we preserve this correlation between $A$ and $C$.  However, when we only require local consistency for the new clique, we will assume that $A$ and $C$ are independent, and lose the correlation between $A$ and $C$.  

Note that all methods described thus far suffer from this problem, not just the one method we evaluated in this paper.  To correctly preserve the correlation between $A$ and $C$, we would have to first estimate the $\set{A,B,C}$ marginal then derive the $\set{A,C}$ marginal from it.  This could be accomplished by adding an $\set{A,B,C}$ region to the region graph and using any of the methods described above.   In this toy problem, it is easy enough to do and feasible, but for larger region graphs, it is not immediately obvious how to generalize the idea.

Since exact marginal inference is not feasible for the graphs we are interested in, it is clear that we must make some approximation.  It is not clear what the nature of the approximation should be, however.  We showed that only using local information in the approximation has problems in some cases, and utilizing some global information may give better results in some cases.  We leave this as an interesting open problem. 

\section{Additional Experiments} \label{sec:extraexp}

\subsection{Synthetic Data used in Experiments} \label{sec:synthetic}

Given a domain size $(n_1, \dots, n_d)$ and a number of records $m$, we generate synthetic data to use in experiments as follows:

\begin{enumerate}
\item Compute a random spanning tree of the complete graph with nodes $1, \dots, d$.  The edges in this tree will correspond to the cliques in our model.  
\item For each edge $ r $ in the tree, set $\btheta_{r} \sim N(0, \sigma^2)^{n_r} $.   Here $\sigma$ is a ``temperature'' parameter that determines the strength of the parameters.  
\item Sample $m$ records from the graphical model $\p_{\btheta}$.  
\end{enumerate}

\subsection{MWEM Experiments} \label{sec:extramwem}

In \cref{fig:mwem} of \cref{sec:experiments} we observed that integrating \apgm into MWEM can enable the mechanism to run for more rounds, but the approximation resulted in much worse error for the same number of rounds.  When run to completion, MWEM+\apgm did achieve lower error than the minimum error achieved by MWEM+\pgm, but it required running for $3\times$ as many rounds and thus spending $3\times$ as much privacy budget.  In the figure below, we include additional lines for different privacy levels $\epsilon=0.05, 0.1, 0.2$ \emph{per round}.  As shown in the figure, it would be better to run MWEM + \mysf{Exact} for 35 round at $\epsilon = 0.1$ than it would be to run MWEM + \mysf{Region Graph} for 70 rounds at $\epsilon = 0.05$.

\begin{center}
\includegraphics[width=0.5\textwidth]{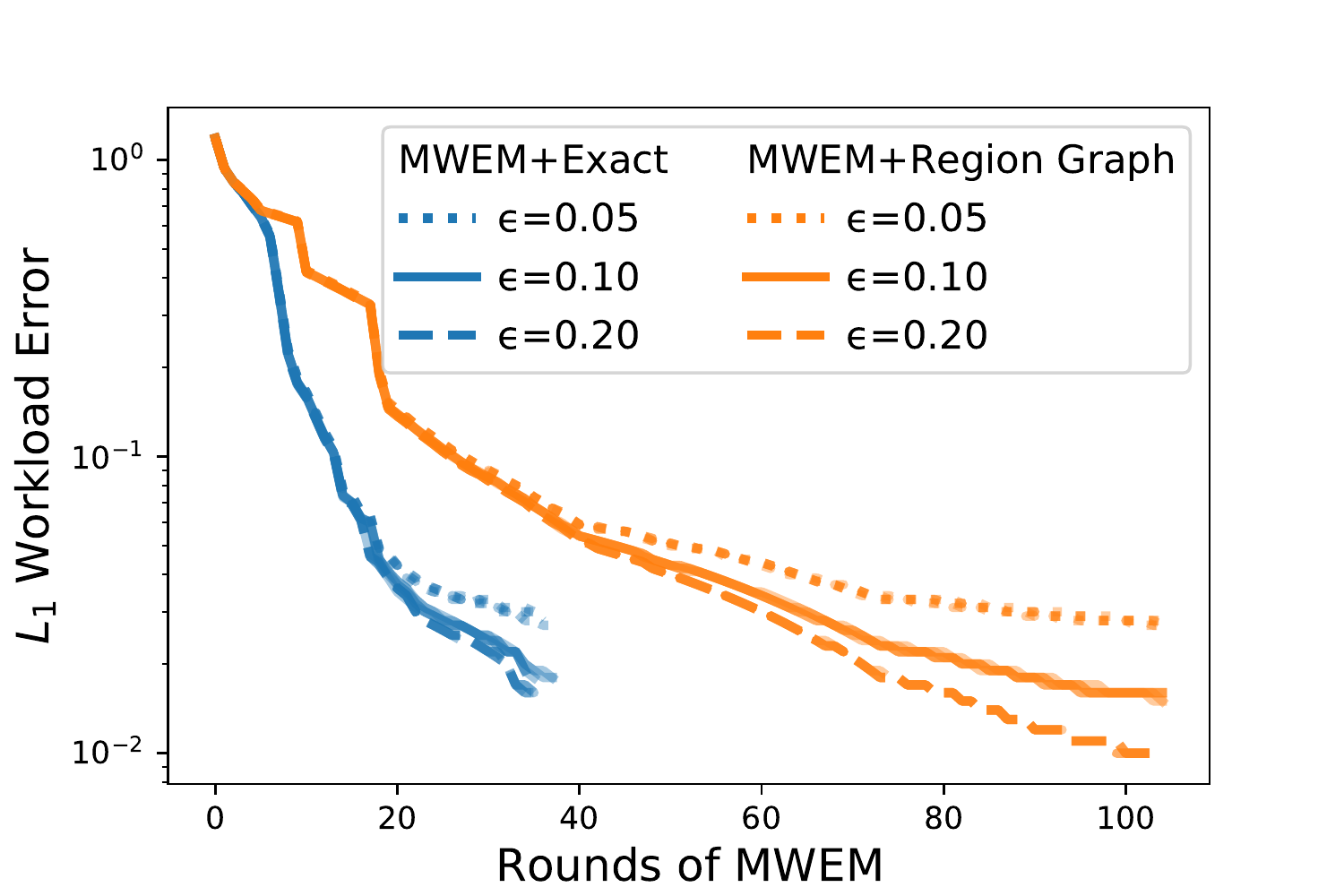}
\end{center}

As hinted at in the main text, the main reason \mysf{Region Graph} performs poorly here is because it only incorporates local information when conducting out-of-model inference, which is problematic for this dataset.  To demonstrate that this is really the problem, we repeat the experiment with $\epsilon=\infty$.  That is, in each round of MWEM, we exactly select the worse approximated clique, and measure the corresponding marginal with no noise.  Since no noise is added, the measured marginals solve \cref{prob:convex} and there is no need to run \prox.  Thus, the only difference between \mysf{Exact} and \mysf{Region Graph} is in the handling of out-of-model marginals.  We run the experiment for five datasets and plot the results below.  The additional error for \mysf{Region Graph} is particularly large for the fire and msnbc dataset but not as much for adult, loans, and titanic.  msnbc is a click stream dataset and is thus naturally modeled as a Markov chain.  Once the 2-way marginals corresponding to the edges in this Markov chain are measured, MWEM + Exact achieves very low error.  MWEM + \mysf{Exact} preserves the long range dependencies between the first and last node in the chain, whereas MWEM + \mysf{Region Graph} only preserves the local dependencies, which explains the difference in this case.  Some datasets (like adult, loans, and titanic) do not have strong dependency chains as msnbc does, and in these cases there is a smaller difference in error for out-of-model marginals.  

\includegraphics[width=0.33\textwidth]{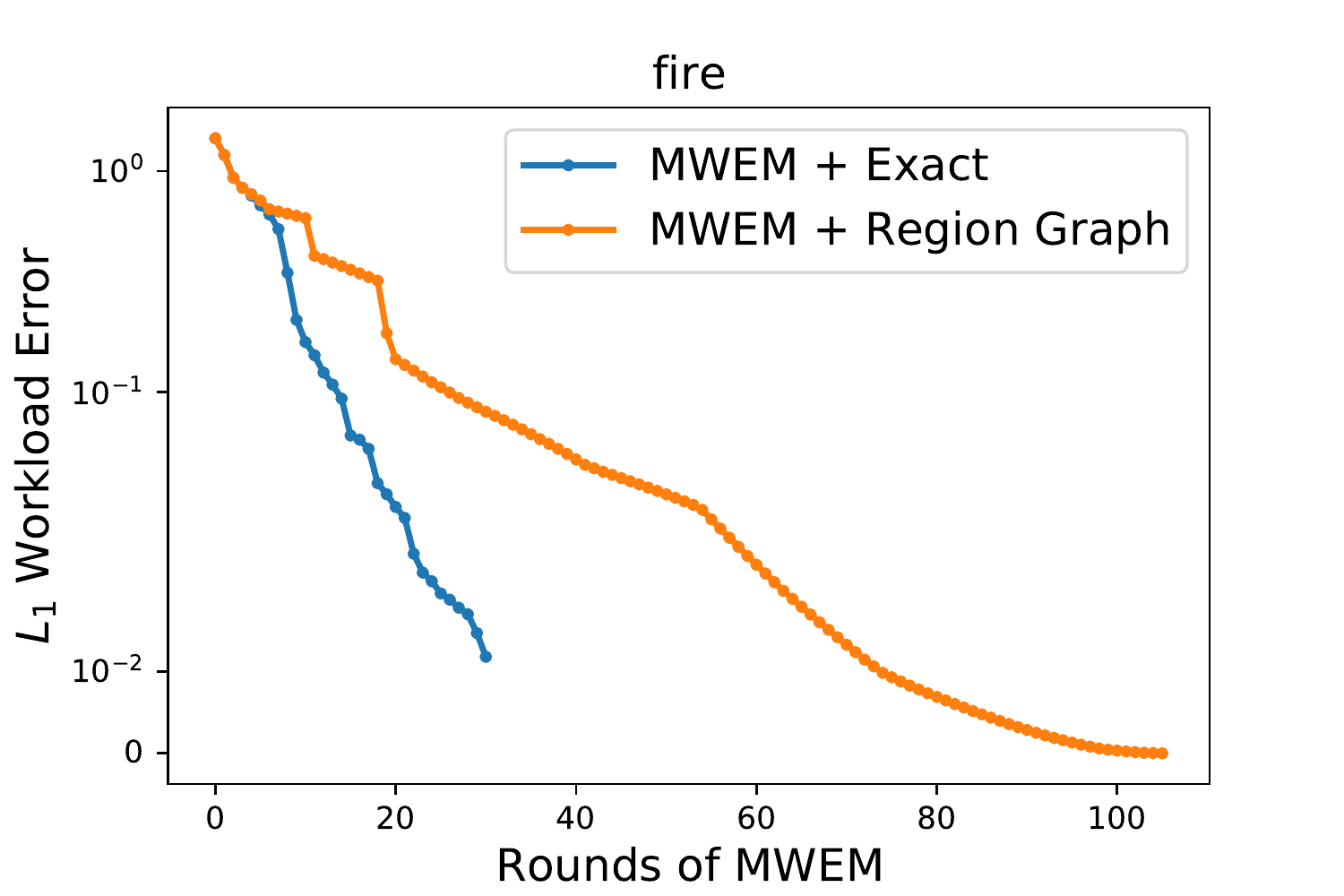}
\includegraphics[width=0.33\textwidth]{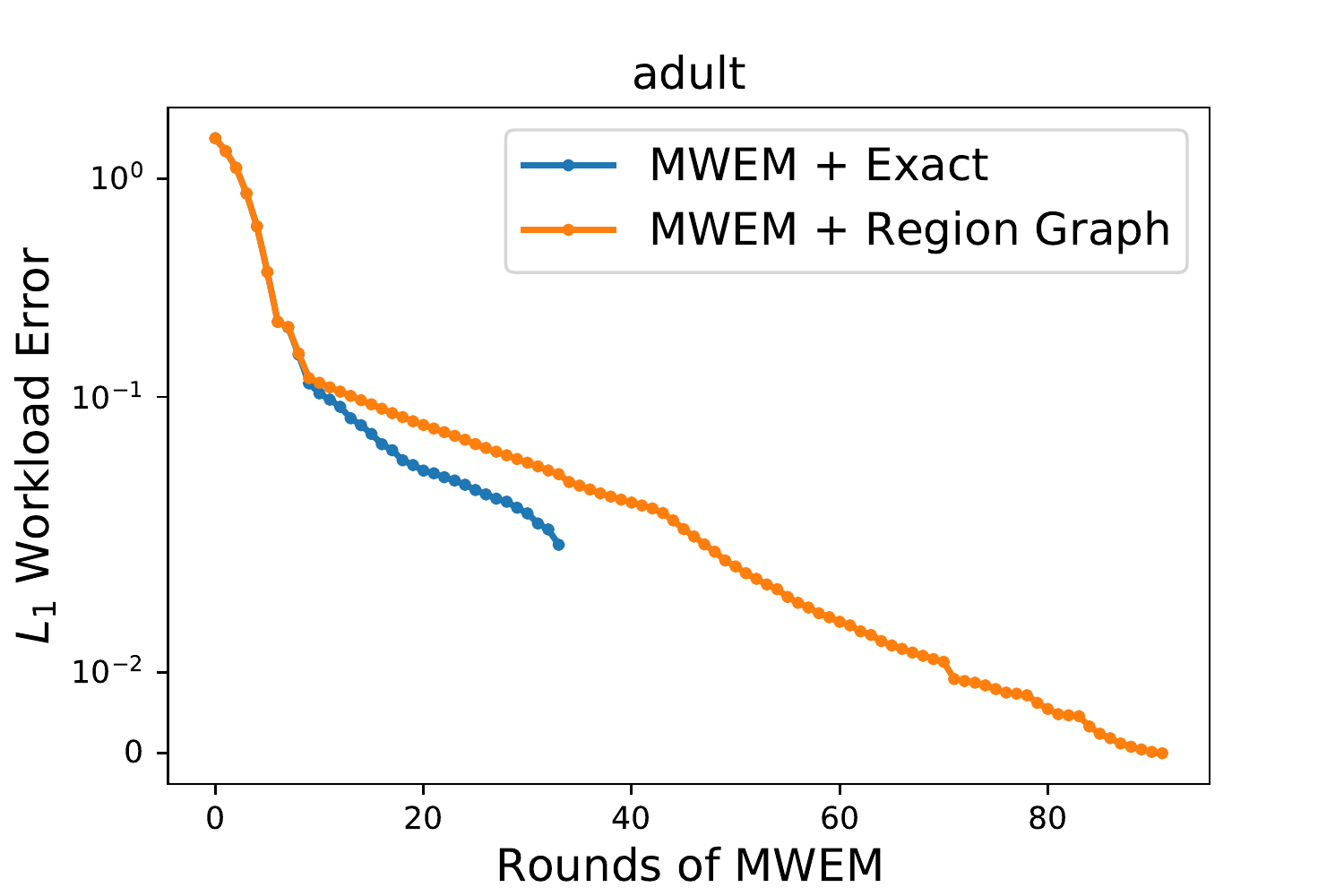}
\includegraphics[width=0.33\textwidth]{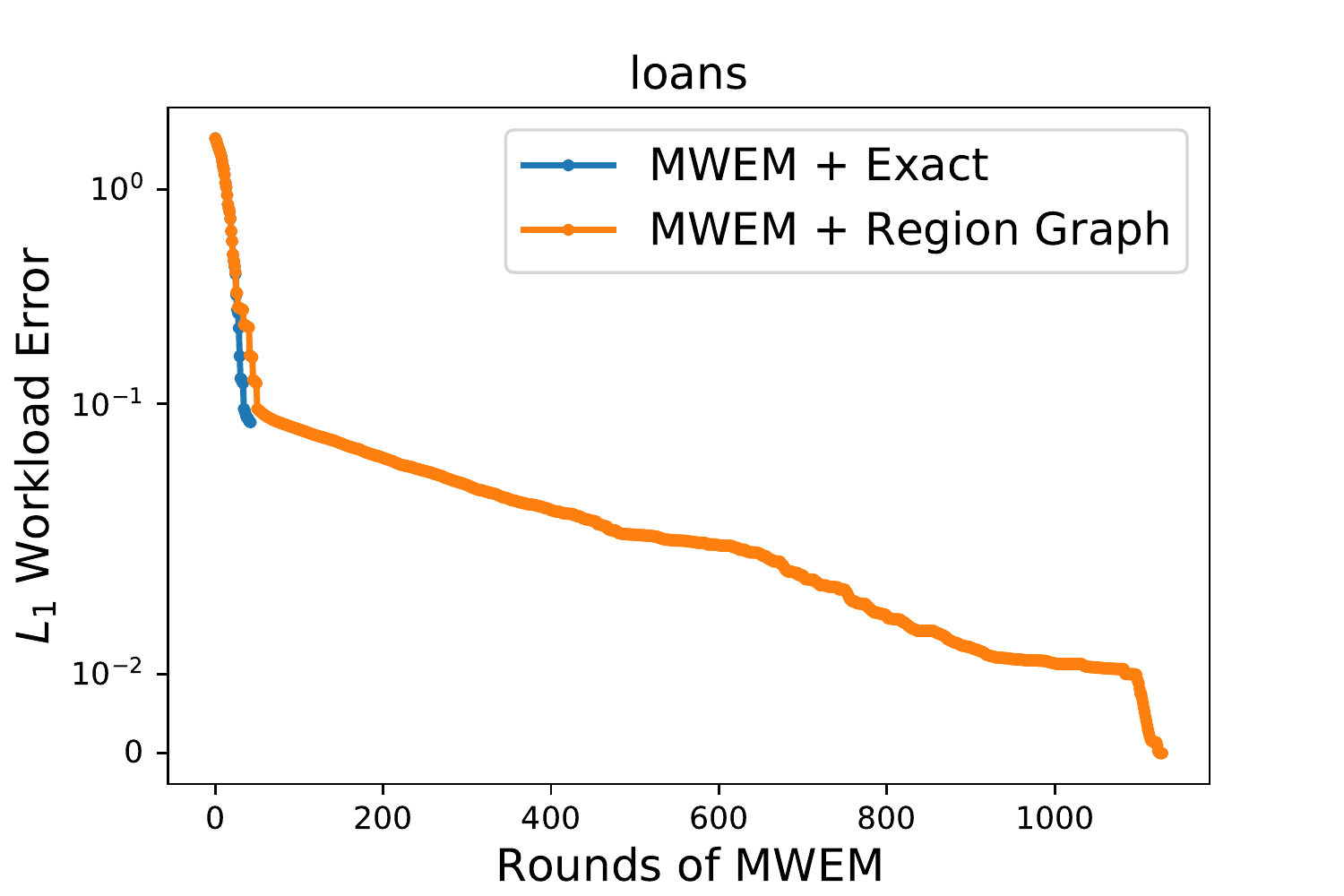}
\begin{center}
\includegraphics[width=0.33\textwidth]{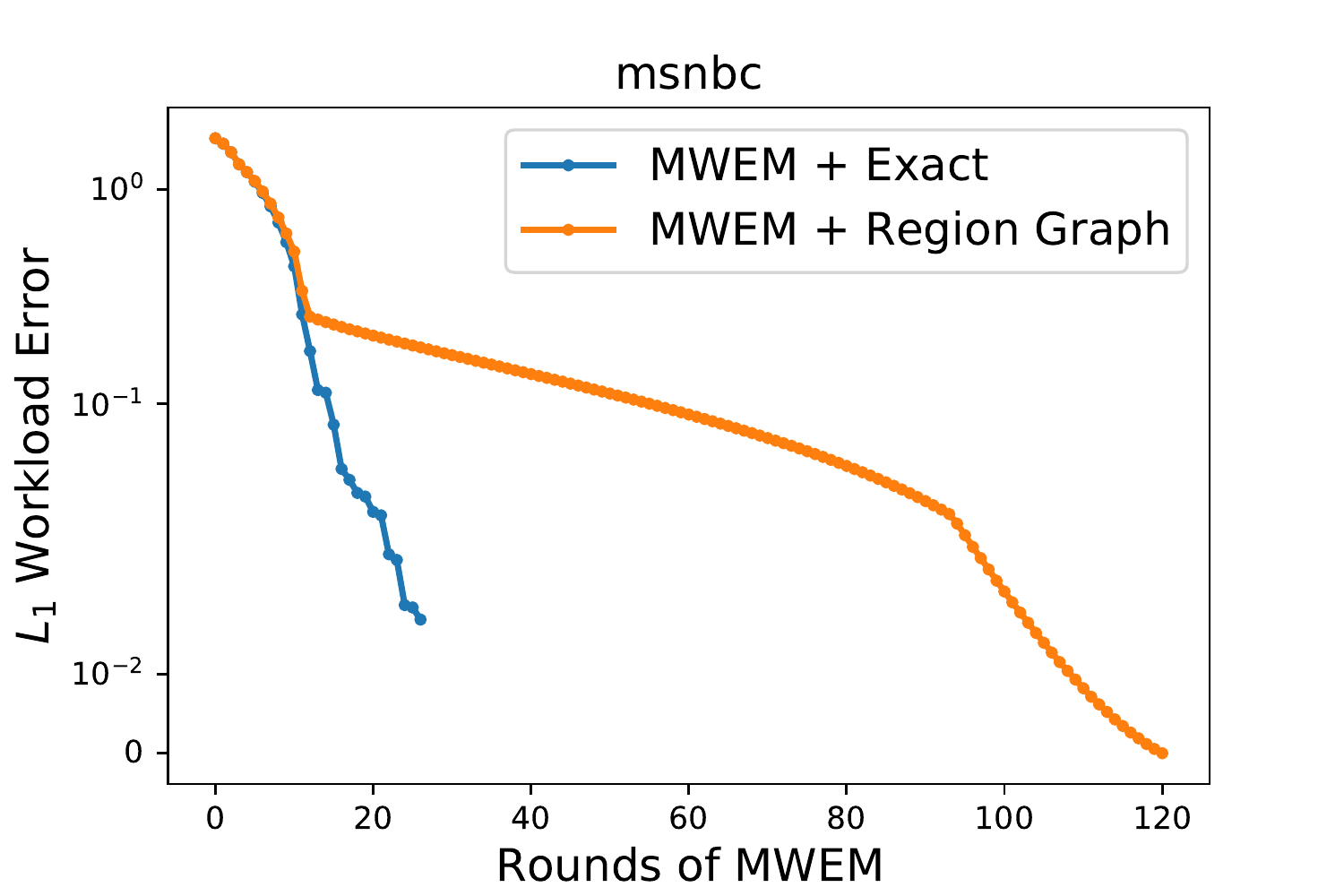}
\includegraphics[width=0.33\textwidth]{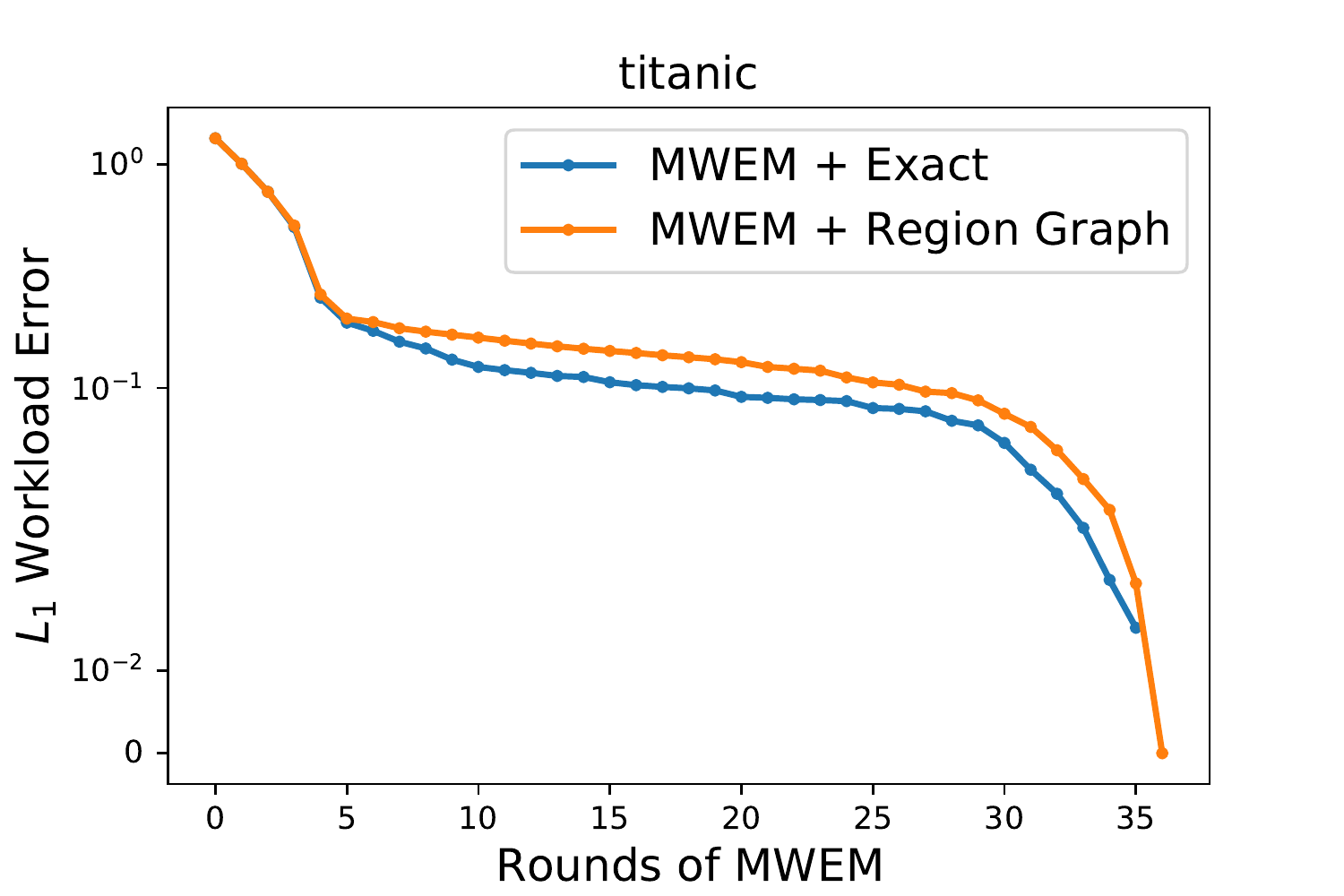}
\end{center}

\subsection{FEM Experiments}

As hinted at in \cref{sec:experiments}, our method can be integrated into FEM as well \cite{vietri2020new}.  The integration is similar to how \pgm is used to improve DualQuery \cite{mckenna2019graphical,Gaboardi14Dual}.  Like MWEM, FEM runs for a specified number of rounds, and maintains an estimate of the data distribution (in tabular format) throughout the execution.  In each round, FEM selects a query from the workload that is poorly approximated under the current estimate of the data distribution using the exponential mechanism.  This is the only way in which FEM interacts with the sensitive data (Unlike MWEM, it does not measure this query with Laplace noise).  It then adds records to the estimated dataset that could explain the previous measurement, in the hopes of reducing the error on that query.  

To integrate into FEM, we first note that the mechanism only depends on the data through the answers to the workload.  If the workload consists of marginal queries, then our methods apply.  We derive an expression for the (negative) log-likelihood of the observations, which are the samples from the exponential mechanism in each round, and use this as our objective function for \cref{prob:convex}.  After solving \cref{prob:convex} with \prox, we can use the estimated pseudo-marginals to answer the workload in place of the synthetic dataset generated by FEM.

In this experiment, we use the adult dataset, as that was one of the main datasets considered in FEM.  We note that FEM has a number of hyper-parameters, and it is not obvious how to select them, and selecting them incorrectly can result in very poor performance.  However, in the authors open source implementation, they provided a set of tuned hyper-parameters a particular dataset/workload pair: the adult dataset and the workload of 64 random 3-way marginals.  For a fair comparison, this is the experimental setting we consider. 

We run FEM and FEM + Region Graph and note that FEM + Exact failed to run here, because the underlying junction tree necessary to perform exact marginal inference is too large.  We report the $L_{\infty}$ workload error (which is what FEM is designed to minimize), as well as the $L_1$ workload error (which better captures the overall error.  The results are shown below.  In general, FEM + Region Graph achieves slightly lower error than regular FEM in both $L_{\infty}$ and $L_1$ error. There is one outlier for $L_{\infty}$ error when $\epsilon=1$ that skews the results, and there was negligible improvement at $\epsilon=0.25$ and $\epsilon=0.5$ as well.  There was consistent improvement in $L_1$ error for every value of $\epsilon$, although the magnitude of the improvement is somewhat small. 

\begin{center}
\includegraphics[width=0.4\textwidth]{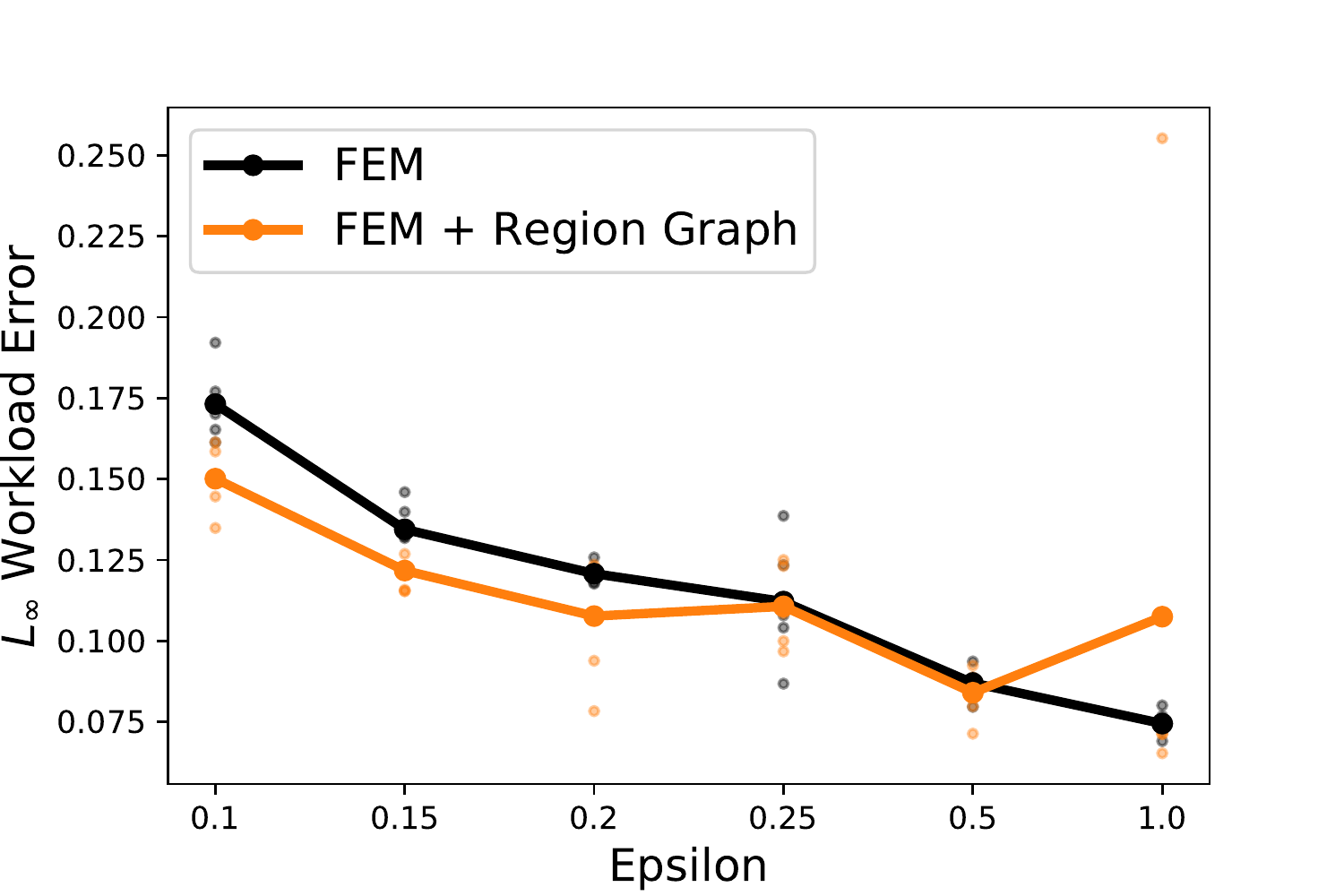} 
\includegraphics[width=0.4\textwidth]{fig/fem_l1} 
\end{center}

\subsection{Comparison with PriView and Relaxed Projection} \label{sec:compare_alternatives}

As discussed in \cref{sec:related}, PriView proposed a method for resolving inconsistencies in noisy marginals that can be seen as a less general competitor to us.  We compare against that competitor here.  We use the implementation of this method available from team DPSyn in the 2018 NIST synthetic data competition \cite{wang2021dpsyn}.  In addition, we compare against a variant of the Relaxed Projection algorithm from \cite{aydore2021differentially}.  We describe the modifications made to this algorithm in the next section.  

To compare these methods with our proposed method, we consider the adult dataset and measure $32$ random $2$-way marginals using the Gaussian mechanism with privacy parameters $\epsilon \in [0.01, 100]$ and $\delta = 10^{-6}$.  In this particular case, \pgm can also run, so we include that as a competitor as well.  We report the $L_1$ error of the estimated marginals, averaged over all measured marginals and $5$ trials for each method in the table below.  All four methods for resolving inconsistencies provide significantly better error than the original noisy marginals.  

Ignoring Relaxed Projection, \prox (Exact) is the best method in every setting except $\epsilon = 100.0$.  The second best method is \prox (Region Graph) in every setting except $\epsilon = 0.01$ and $\epsilon = 100.0$.  At the smallest value of $\epsilon$, our method is likely overfitting to the noise, and the estimated pseudo-marginals are likely far from the set of realizable marginals.  At the largest value of $\epsilon$, both variants of \prox simply didn't run for enough iterations ($10000$ was used in this experiment).  Due to the small amount of noise, the true solution to \cref{prob:convex} likely does not contain any negatives, and the PriView approach solves the relaxed problem without the non-negativity constraints in closed form.  \prox should eventually converge to the same solution but it would require more than 10000 iterations.  

Relaxed Projection (RP) performs slightly better than even \prox (Exact) for $\epsilon \leq 0.1$, an interesting and surprising observation.  We conjecture that this is because RP essentially restricts the search space to distributions which are a mixture of products (as described in the next section).  This can be seen as a form of regularization, which can help in the high-privacy / high-noise regime.  For $\epsilon > 0.1$, RP is worse than both \prox (Exact) and \prox (Region Graph).  Moreover, it is the only method whose error does not tend towards $0$ as $\epsilon$ gets larger.  We suspect this is due to the non-convexity in the problem formulation for RP: it is finding a local minimum to the problem that does not have $0$ error.  Alternatively, it could be possible that the restircted search space does not include a distribution with near-zero error, although we believe this is a less likely explanation.  

\begin{center}
\resizebox{\textwidth}{!}{
\begin{tabular}{c|ccccc}
\toprule
$\epsilon$ & \prox & \prox & PriView & Relaxed & Noisy \\
& (Exact) & (Region Graph) & Consistency & Projection & Marginals \\
\midrule
\texttt{\footnotesize{0.0100}}                 &  \texttt{\footnotesize{0.4375 $\pm$ 0.0245}} &   \texttt{\footnotesize{0.5630 $\pm$ 0.0344}} &  \texttt{\footnotesize{0.5229 $\pm$ 0.0202}} &  \texttt{\footnotesize{0.4189 $\pm$ 0.0275}} &  \texttt{\footnotesize{28.050 $\pm$ 0.1249}} \\
\texttt{\footnotesize{0.0316}} &  \texttt{\footnotesize{0.2848 $\pm$ 0.0081}} &    \texttt{\footnotesize{0.3277 $\pm$ 0.0100}} &  \texttt{\footnotesize{0.3525 $\pm$ 0.0078}} &  \texttt{\footnotesize{0.2567 $\pm$ 0.0045}} &   \texttt{\footnotesize{8.8782 $\pm$ 0.0254}} \\
\texttt{\footnotesize{0.1000}}                  &  \texttt{\footnotesize{0.1724 $\pm$ 0.0032}} &  \texttt{\footnotesize{0.1788 $\pm$ 0.0025}} &  \texttt{\footnotesize{0.1965 $\pm$ 0.0051}} &   \texttt{\footnotesize{0.1620 $\pm$ 0.0036}} &   \texttt{\footnotesize{2.8091 $\pm$ 0.0101}} \\
\texttt{\footnotesize{0.3162}}  &  \texttt{\footnotesize{0.0908 $\pm$ 0.0009}} &  \texttt{\footnotesize{0.0931 $\pm$ 0.0018}} &  \texttt{\footnotesize{0.1007 $\pm$ 0.0016}} &  \texttt{\footnotesize{0.1031 $\pm$ 0.0025}} &    \texttt{\footnotesize{0.8919 $\pm$ 0.0030}} \\
\texttt{\footnotesize{1.0000}}                  &  \texttt{\footnotesize{0.0433 $\pm$ 0.0008}} &  \texttt{\footnotesize{0.0447 $\pm$ 0.0006}} &   \texttt{\footnotesize{0.0510 $\pm$ 0.0003}} &  \texttt{\footnotesize{0.0746 $\pm$ 0.0009}} &   \texttt{\footnotesize{0.2853 $\pm$ 0.0007}} \\
\texttt{\footnotesize{3.1622}}   &  \texttt{\footnotesize{0.0187 $\pm$ 0.0001}} &  \texttt{\footnotesize{0.0198 $\pm$ 0.0002}} &  \texttt{\footnotesize{0.0229 $\pm$ 0.0003}} &  \texttt{\footnotesize{0.0617 $\pm$ 0.0007}} &   \texttt{\footnotesize{0.0934 $\pm$ 0.0003}} \\
\texttt{\footnotesize{10.000}}                 &  \texttt{\footnotesize{0.0074 $\pm$ 0.0001}} &  \texttt{\footnotesize{0.0087 $\pm$ 0.0001}} &  \texttt{\footnotesize{0.0095 $\pm$ 0.0001}} &  \texttt{\footnotesize{0.0582 $\pm$ 0.0011}} &   \texttt{\footnotesize{0.0324 $\pm$ 0.0001}} \\
\texttt{\footnotesize{31.622}}   &     \texttt{\footnotesize{0.0037 $\pm$ 0.0000}} &     \texttt{\footnotesize{0.0045 $\pm$ 0.0000}} &      \texttt{\footnotesize{0.0040 $\pm$ 0.0000}} &  \texttt{\footnotesize{0.0579 $\pm$ 0.0017}} &      \texttt{\footnotesize{0.0125 $\pm$ 0.0000}} \\
\texttt{\footnotesize{100.00}}                &     \texttt{\footnotesize{0.0027 $\pm$ 0.0000}} &     \texttt{\footnotesize{0.0032 $\pm$ 0.0000}} &     \texttt{\footnotesize{0.0018 $\pm$ 0.0000}} &  \texttt{\footnotesize{0.0574 $\pm$ 0.0012}} &      \texttt{\footnotesize{0.0054 $\pm$ 0.0000}} \\
\bottomrule
\end{tabular}}
\end{center}

\subsection{Implementation Details for Relaxed Projection}

The authors of the Relaxed Projection method released their code on GitHub.  They provided code to run their end-to-end MWEM-style algorithm, but did not expose the subroutine for performing the relaxed projection in a way that can easily be tested in isolation.  For that reason, we compare against a faithful reimplementation of their approach.  This reimplementation is available in the open source \pgm repository.  

One way to view RP is as optimizing over the set of distributions which are mixtures of products.  That is, each row of the relaxed tabualr format can be viewed as a product distribution (if the values for each feature are non-negative and sum to one).  For multiple rows, this translates to a format that has capacity to represent a mixture of product distributions.  While the authors do not propose restricting the feature values to satisfy the aforementioned constraints, in our reimplemenation, we apply softmax transformations to the table to ensure this invariant holds.  This is related to RAP\textsuperscript{softmax} as described by Liu et al. \cite{liu2021iterative}, although the interpretation as a mixture of products was not mentioned in that work.  For the experiment above, we consider distributions with $100$ mixture components.  

\newpage

\end{document}